\RequirePackage{fix-cm}
\documentclass[smallextended]{svjour3}       
\smartqed  

\usepackage{amssymb}

\usepackage[unicode=true,
linktocpage,
linkbordercolor={0.5 0.5 1},
citebordercolor={1 1 1},
linkcolor=blue]{hyperref}

\usepackage{url}
\usepackage{graphicx}
\usepackage{xcolor}

\usepackage[round]{natbib}

\usepackage{enumerate}
\usepackage[tight,footnotesize]{subfigure} 

\usepackage{booktabs}
\usepackage{epstopdf}
\usepackage{multirow}

\usepackage{amsmath,amssymb}
\allowdisplaybreaks
\usepackage{algorithmic,algorithm}

\usepackage{alphalph}

\newcommand{\xopt}{x^{*}}
\newcommand{\fbest}{y^{*}}

\newcommand{\xt}{x_{t}}
\newcommand{\fu}{f^{u}_{t-1}}
\newcommand{\fl}{f^{l}_{t-1}}
\newcommand{\sigmat}{\sigma_{t-1}}
\newcommand{\mut}{\mu_{t-1}}
\newcommand{\bt}{\beta_{t}^{1/2}}

\definecolor{blu}{rgb}{0,0,1}
\definecolor{gre}{rgb}{0,.5,0}

\newcommand{\reb}[1]{\textcolor{black}{#1}}

\newcommand{\mikecommenta}[1]{\textcolor{black}{#1}}
\def\blu#1{{\color{black}#1}}
\def\ylow#1{{\color{black}#1}}
\def\gre#1{{\color{black}#1}}

\makeatletter
\def\thm@space@setup{%
  \thm@preskip=\parskip \thm@postskip=0pt
}
\makeatother

\renewcommand{\d}[1]{\ensuremath{\operatorname{d}\!{#1}}}

\newcommand{\R}{\mathbb{R}}

\newcommand{\E}{\mathbb{E}}

\newcommand{\norm}[1]{\left\|#1\right\|}

\newcommand{\cD}{\mathcal{D}}

\newcommand{\cP}{\mathcal{P}}

\newcommand{\cX}{\mathcal{X}}

\newcommand{\bk}{{\bf k}}

\newcommand{\by}{{\bf y}}

\renewcommand{\epsilon}{\varepsilon}
\renewcommand{\hat}{\widehat}
\renewcommand{\tilde}{\widetilde}
\renewcommand{\bar}{\overline}

\newcommand{\nothere}[1]{}

\usepackage{xspace}

\usepackage[colorinlistoftodos, textwidth=30mm, shadow, textsize=small]{todonotes}
\definecolor{babyblue}{rgb}{0.54, 0.81, 0.94}
\definecolor{citrine}{rgb}{0.89, 0.82, 0.04}
\definecolor{misocolor}{rgb}{0.16,0.27,0.86}

\usepackage{etoolbox}

\begin{document}

\title{Combining Bayesian Optimization and Lipschitz Optimization}
\author{Mohamed Osama Ahmed  \and \\
	Sharan Vaswani  \and \\
	Mark Schmidt
}

\institute{Mohamed Osama Ahmed \at
	\email{moahmed@cs.ubc.ca}           
	\and
	Sharan Vaswani \at
	\email{vaswanis@mila.quebec} 
	\and
	Mark Schmidt \at
	\email{schmidtm@cs.ubc.ca} \\
}

\maketitle

\begin{abstract}
Bayesian optimization and Lipschitz optimization have developed alternative techniques for optimizing black-box functions. They each exploit a different form of prior about the function. In this work, we explore strategies to combine these techniques for better global optimization. In particular, \blu{we propose ways to use the Lipschitz continuity assumption within traditional BO algorithms, which we call Lipschitz Bayesian optimization (LBO). This approach does not increase the asymptotic runtime and in some cases drastically improves the performance (while in the worst-case the performance is similar).} Indeed, in a particular setting, we prove that using the Lipschitz information yields the same or a better bound on the regret compared to using Bayesian optimization on its own. Moreover, we propose a simple heuristics to estimate the Lipschitz constant, and prove that a growing estimate of the Lipschitz constant is in some sense ``harmless''. Our experiments on 15 datasets with 4 acquisition functions show that in the worst case LBO performs similar to the underlying BO method while in some cases it performs substantially better. Thompson sampling in particular typically saw drastic improvements (as the Lipschitz information corrected for its well-known ``over-exploration'' phenomenon) and its LBO variant often outperformed other acquisition functions.
\end{abstract}
\section{Introduction}
Bayesian optimization (BO) has a long history and has been used in a variety of fields~\citep[see][]{shahriari2016taking}, with recent interest from the machine learning community in the context of automatic hyper-parameter tuning~\citep{snoek2012practical,golovin2017google}. BO is an example of a global black-box optimization algorithm~\citep{hendrix2010introduction,jones1998efficient,pinter1991global,rios2013derivative} which optimizes an unknown function that may not have nice properties such as convexity. In the typical setting, we assume that we only have access to a black box that evaluates the function and that it is expensive to do these evaluations. The objective is to find a global optimum of the unknown function with the minimum number of function evaluations. 

\gre{The global optimization of a real-valued function} is impossible unless we make assumptions about the structure of the unknown function. \reb{Lipschitz continuity assumes that the function can't change arbitrarily fast as we change the inputs.} This is one of the weakest assumptions under which optimizing an unknown function is still possible.
Lipschitz optimization~\citep{piyavskii1972algorithm,shubert1972sequential} (LO) exploits knowledge of a bound on the Lipschitz constant $L$ of the function. \reb{This constant $L$ specifically gives a bound on the maximum amount that the function can change (as the parameters change)}. This bound allows LO to prune the search space in order to locate the optimum. In contrast, Bayesian optimization, makes the assumption that the unknown function belongs to a \blu{known model class (typically a class of smooth functions)}, the most common being a Gaussian process (GP) \gre{generated using a Gaussian or Mat\'{e}rn kernel~\citep{stein2012interpolation}}. We review LO and BO in Section~\ref{sec:back}.

Under their own specific sets of additional assumptions, both BO~\citep[Theorem~5]{bull2011convergence} and LO~\citep{malherbe2017global} can be shown to be exponentially faster than random search strategies. If the underlying function is close to satisfying the stronger BO assumptions, then \gre{typically BO is able to optimize functions faster} than LO. However, when these assumptions are not reasonable, BO may converge slower than simply trying random values~\citep{li2016efficient,ahmed2016we}. On the other hand, LO makes minimal assumptions (not even requiring differentiability\footnote{The absolute value function $f(x) = |x|$ is an example of a simple non-differentiable but Lipschitz-continuous function.}) and simply prunes away values of the parameters that are not compatible with the Lipschitz condition and thus cannot be solutions. This is useful in speeding up simple algorithms like random search. Given a \gre{new function to optimize}, it is typically not clear which of these strategies will perform better. 

\blu{In this paper, we propose to combine BO and LO to exploit the advantages of both methods. We call this \emph{Lipschitz Bayesian Optimization} (LBO). Specifically, in Section~\ref{sec:lbo}, we design \emph{mixed acquisition functions} that use Lipschitz continuity in conjunction with existing BO algorithms. We also address the issue of providing a ``harmless''  estimate of the Lipschitz constant (see Section~\ref{sec:L-estimation}), which is an important practical issue \gre{ for any LO method}. Our experiments (Section~\ref{sec:experiments}) indicate that in some settings the addition of estimated Lipschitz information leads to a huge improvement over standard BO methods. This is particularly true for Thompson sampling, which often outperforms other standard acquisition functions when augmented with Lipschitz information. This seems to be because the estimated Lipschitz continuity seems to correct for the well-known problem of over-exploration~\citep{shahriari2014entropy}. Further, our experiments indicate that it does not hurt to use the Lipschitz information since even in the worst case it does not change the runtime or the performance of the method.}

\section{Background}
\label{sec:back}
We consider the problem of maximizing a real-valued function \gre{$f$ with parameters $x$} over a compact set $\cX$. 
We assume that on iteration $t$, 
an algorithm chooses a point $x_t \in \cX$ and then receives the corresponding function value $f(x_t)$. Typically, our goal is to find the largest possible $f(x_t)$ across iterations. We describe two approaches for solving this problem, namely BO and LO, in detail below.

\subsection{Bayesian Optimization}
\label{sec:bo}
BO methods are typically based on Gaussian processes (GPs), since they have appealing universal consistency properties \gre{over compact sets} and admit a closed-form posterior distribution~\citep{rasmussen2006gaussian}. BO methods typically assume a smooth GP prior on the unknown function, and use the observed function evaluations to compute a posterior distribution over the possible function values at any point $x$. At iteration $t$, given the previously selected points $\{x_{1}, x_{2}, \ldots x_{t-1} \}$ and their corresponding observations $\by_{t} = [y_{1}, y_{2}, \ldots, y_{t-1}]$, the algorithm uses an \emph{acquisition function} (based on the GP posterior) to select the next point to evaluate. The value of the acquisition function at a point characterizes the importance of evaluating that point in order to maximize $f$. To determine $x_t$, we \gre{usually} maximize this acquisition function over all $x$ using an auxiliary optimization procedure (typically we can only approximately solve this maximization~\citep{wilson2018maximizing, kim2019local}).

We now formalize the above high-level procedure. We assume that \gre{$f$ follows a  $GP(0,k(x,x^\prime))$ distribution where} $k(x,x^\prime)$ is a kernel function which quantifies the similarity between points $x$ and $x^\prime$. Throughout this paper, we use the Mat\'{e}rn kernel for which $k(x,x^\prime) = 
\reb{{\sigma_0}^{2}} \exp \left( - \sqrt{5} r^{2} \right) \left( 1 + \sqrt{5}r + \frac{5 r^{2}}{3} \right)$
where $r = \sum_{j=1}^d \frac{(x_j-x^\prime_j)^2} {\ell_j}$. Here the hyper-parameter $\ell_j$ is referred to as the  length-scale for dimension $j$ and dictates the extent of smoothness we assume about the function $f$ in direction $j$.  \reb{The hyper-parameter $\sigma_0$ represents the scale of the signal.}

We denote the maximum value of the function until iteration $t$ as $y^{*}_{t}$ and the set $\{ 1,2, \ldots , t \}$ as $[t]$. Let $\bk_{t}(x) = [k(x,x_{1}), k(x,x_{2}), \ldots, k(x,x_{t})]$ and let us denote the $t \times t$ kernel matrix as $K$ (so $K_{i,j} = k(x_{i}, x_{j})$ for all $ i,j \in [t]$). Given the function evaluations (observations), the posterior distribution at point $x$ after $t$ iterations is given as $\cP(f_{t}(x)) \reb{=} N(\mu_{t}(x),{\sigma_{t}}^{\reb{2}}(x))$. Here, the mean and standard deviation of the function at $x$ are given as:
\begin{align}
\mu_{t}(x) & = \bk_{t}(x)^{T} \left( K + \sigma^{2} I_{t} \right)^{-1} \by_{t}, \nonumber \\
{\sigma_{t}}^{\reb{2}}(x) & = k(x,x) - \bk_{t}(x)^{T} \left( K + \sigma^{2} I_{t} \right)^{-1} \bk_{t}(x). \label{eq:gp-posterior}
\end{align}
As alluded to earlier, an acquisition function uses the above posterior distribution in order to select the next point to evaluate the function at. A number of acquisition functions have been proposed in the literature, with the most popular ones: (UCB)~\citep{srinivas2009gaussian}, Thompson sampling (TS)~\citep{thompson1933likelihood}, expected improvement (EI)~\citep{movckus1975bayesian}, probability of improvement (PI)~\citep{kushner1964new}, and entropy search~\citep{villemonteix2009informational, hennig2012entropy, hernandez2014predictive}. In this work, we focus on four simple widely-used acquisition functions: 
UCB, TS, EI, and  PI.
However, we expect that our conclusions would apply to other acquisition functions. For brevity, when defining the acquisition functions, we drop the $(t-1)$ subscripts from $\mu_{t-1}(x)$, $\sigma_{t-1}(x)$, and $y^{*}_{t-1}$ .\\
\textbf{UCB}: The acquisition function $UCB(x)$ is defined as:
\begin{align}
UCB(x) = \mu(x) + \beta_{t}^{1/2} \sigma(x). \label{eq:UCB-def}
\end{align}
Here, $\beta_{t}$ is positive parameter that trades off exploration and exploitation. \\
\textbf{TS}: For TS, in each iteration we first sample a function $\tilde{f}_{t}(x)$ from the GP posterior\gre{, $\tilde{f}_{t} \sim GP(\mu_{t}(x), \sigma_{t}(x))$}. TS then selects the point $x_{t}$ which maximizes this deterministic function \gre{$\tilde{f}_{t}$}. \\
\textbf{PI}: We define the possible improvement (over the current maximum) at $x$ as \mikecommenta{ $I(x) = \max \{ \reb{f_x(x)} - y^{*}, 0\}$  and the indicator of improvement $u(x)$ as  \[
	u(x) =  
	\begin{cases}
	0, & \text{if } \reb{f_x(x)}  < y^* \\
	1,  & \text{if } \reb{f_x(x)}  \geq y^*
	\end{cases},
	\]} \reb{where $f_x(x) \sim$ $P(f_x)$}.
PI selects the point $x$ which maximizes 
the probability of improving over $y^{*}$. If \mikecommenta{ $\phi(\cdot)$ and $\Phi(\cdot)$ are the probability density function and the cumulative distribution function for the standard normal distribution}, then the PI acquisition function is given as~\citep{kushner1964new}:
\begin{equation}\label{eq:PI-def}
\begin{split}
PI(x) & = \int_{-\infty}^{\infty} u(x)\phi(f_x(x))\mathrm{d}f_x 
= \Phi \left( z(x, \fbest) \right).
\end{split}
\end{equation}
where we have defined 
$z(u,v) = \frac{\mu(u) - v}{\sigma(u)}$.\\
\textbf{EI}: EI selects an $x$ that maximizes $\E[I(x)]$, where the expectation is over the distribution $\cP(f_{t}(x))$. If $\phi(\cdot)$ is the pdf of the standard normal distribution, the expected improvement acquisition function can be written as~\citep{movckus1975bayesian}:
\begin{equation}\label{eq:EI-def}
\begin{split}
EI(x) & = \int_{-\infty}^{\infty} I(x)\phi(f_x(x))\mathrm{d}f_x \\
& = \int_{y^*}^{\infty} (f_x(x)-y^*)\phi(f_x(x))\mathrm{d}f_x \\
& =\sigma(x) \cdot \left[z(x, \fbest) \cdot \Phi( z(x,\fbest) ) + \phi( z(x,\fbest) ) \right].
\end{split}
\end{equation}

\subsection{Lipschitz Optimization}
\label{sec:lo}

As opposed to assuming that the function comes from a specific family of functions, in LO we simply assume that the function
cannot change too quickly as we change $x$. In particular, we say that a function $f$ is Lipschitz-continuous if for all $x$ and $x'$ we have
\begin{align}
\vert f(x) - f(x') \vert  \leq  L \vert \vert x - x' \vert \vert_{2}, \label{eq:zero-lipschitz} 
\end{align}
for \gre{ a constant $L$ which is referred to as the Lipschitz constant}. \reb{Note that unlike the typical priors used in BO (like the Gaussian or Mat\'{e}rn kernel), a function can be non-smooth and still be Lipschitz continuous.}

\reb{Lipschitz optimization methods consider the deterministic (noiseless) case, where $y_i = f(x_i)$. In this setting,}
Lipschitz optimization uses this Lipschitz inequality in order to test possible locations for the maximum of the function. In particular, at iteration $t$ the Lipschitz inequality implies that the function's value at any $x$ can be upper and lower bounded for any $i \in [t-1]$ by
\[
f(x_{i}) - L \vert \vert x - x_{i} \vert \vert_{2} \leq f(x) \leq f(x_{i}) + L \vert \vert x - x_{i} \vert \vert_{2}.
\]
Since the above inequality holds simultaneously for all $i \in [t-1]$, for any $x$ the function value $f(x)$ can be bounded as:
\begin{align}
f^{l}_{t-1}(x) \leq f(x) \leq f^{u}_{t-1}(x), \nonumber  \text{ where,}
\end{align}
\begin{align}
f^{l}_{t-1}(x) = \max_{i \in [t-1]} \left\{ f(x_{i}) - L \vert \vert x - x_{i} \vert \vert_{2} \right\} \nonumber \\   
f^{u}_{t-1}(x) = \min_{i \in [t-1]} \left\{ f(x_{i}) + L \vert \vert x - x_{i} \vert \vert_{2} \right\} \label{eq:lip-bounds} 
\end{align}
Notice that if $f^u_{t-1}(x) \leq y_{t-1}^*$, then $x$ \emph{cannot} achieve a higher function value than our current maximum $y_{t-1}^*$.

\blu{To exploit these bounds, at each iteration of a typical Lipschitz optimization (LO) method, \citet{malherbe2017global} \gre{might sample} points $x_p$ uniformly at random from $\cX$ until it finds an $x_p$ that satisfies $f^{u}_{t-1}(x_{p}) \geq \fbest_{t-1}$. If we know the Lipschitz constant $L$ (or use a valid upper bound on the minimum $L$ value), this strategy may prune away large areas of the space while guaranteeing that we do not prune away any optimal solutions.} \gre{This can substantially decrease the number of function values needed to come close to the global optimum compared to using random points without pruning.}

\blu{A major drawback of Lipschitz optimization is that in most applications we do not know a valid $L$. We discuss this scenario in the next section, but first we note \gre{that there exist applications where we do} have access to a valid $L$. For example,~\citet{bunin2016lipschitz} discuss cases where $L$ can be dictated by the physical laws of the underlying process (e.g., in heat transfer,  solid oxide fuel-cell system, and  polymerization).  Alternately, if we have a lower and an upper bound on the possible values that the function can take, then we can combine this with the size of $\mathcal{X}$ to obtain an over-estimate of the minimum $L$ value. }
\subsection{Harmless Lipschitz Optimization}

\label{sec:L-estimation}
When our black-box functions arises from a real world process, a suitable value of $L$ is typically dictated by physical limitations of the process. However, in practice we often do not know $L$ and thus need to  estimate it. A simple way to obtain an under-estimate \gre{${L}_{t}^{lb}$} of $L$ at iteration $t$ is to use the maximum value that satisfies the Lipschitz inequality across all pairs of points,
\begin{align}
\gre{{L}^{lb}_{t}} = \max_{i,j \in [t]; x_i \neq x_j} \left\{\frac{ \vert f(x_{i}) - f(x_{j}) \vert  }{\vert \vert x_{i} - x_{j} \vert \vert_{2}} \right\}.
\label{eq:L0-estimate-lb}
\end{align}
Note that this estimate monotonically increases as we see more examples, but that it may be far smaller than the true $L$ value \reb{(and recall that we are considering the noiseless case where $f(x_i)=y_i$)}.
\blu{A common variation is to sample several points on a grid (or randomly) to use in the estimate above. Unfortunately, without knowing the Lipschitz constant we do not know how fine this grid should be so in general this may still \gre{significantly under-estimate the true quantity.}}

\blu{A reasonable property of any estimate of $L$ that we use is that it is ``harmless'' in the sense of~\citet{ahmed2016we}. Specifically, the choice of $L$ should not make the algorithm converge to the global optimum at a slower speed than random guessing (in the worst case). If we have an over-estimate for the minimum possible value of $L$, then the LO algorithm is harmless as it can only prune values that cannot improve the objective function (although if we over-estimate it by too much then it may not prune much of the space). However, the common under-estimates of $L$ discussed in the previous paragraph are \emph{not} harmless since they may prune the global optima.}

\blu{We propose a simple solution to the problem that LO is not harmless if we don't have prior knowledge about $L$: we use a \emph{growing estimate of $L$}. The danger in using a growing strategy is that if we grow $L$ too slowly then the algorithm may not be harmless. However, in the appendix we show that LO is ``harmless'' for most reasonable strategies for growing $L$. This result is not prescriptive in the sense that \gre{it does not suggest} a practical strategy for growing $L$ (since it depends on the true $L$), but this result shows that even for enormous values of $L$ that an estimate would have to be growing exceedingly slowly in order for it to not be harmless (exponentially-slow  in the minimum value of $L$, the dimensionality, and the desired accuracy). 
}
\mikecommenta{ In our experiments we simply use  $L_{t}^{ub} = \gre{\kappa}  t \cdot L_{t}^{lb}$, the under-estimator multiplied by the \gre{(growing)} iteration number and a constant \gre{$\kappa$} (a tunable hyper-parameter).} In Section~\ref{sec:experiments}, we observe that this choice of $L_{t}^{ub}$ with \gre{$\kappa = 10$} consistently works well  across 14 datasets with 4 different acquisition functions.

\section{Lipschitz Bayesian optimization}
\label{sec:lbo}
In this section, we show how simple changes to the standard acquisition functions used in BO allow us to incorporate the Lipschitz inequality bounds. We call this Lipschitz Bayesian Optimization (LBO).  \gre{LBO}  prevents BO from \gre{considering} values of $x^t$ that  \mikecommenta{cannot} be global maxima \blu{(assuming we have over-estimated $L$)} and also restricts the range of $f(x_t)$ values considered in the acquisition function to those that are consistent with the Lipschitz \gre{inequalities}. 
Figure~\ref{fig:visualization} illustrates the \gre{key features of BO, LO,} and LBO.
\mikecommenta{ It is important to note that \gre{the Lipschitz constant $L$ has a different interpretation than} the length-scale $\ell$ of the GP. The constant $L$ specifies an absolute maximum rate of change for the function, while $\ell$ specifies \gre{how quickly a parameterized distance between pairs of points changes the GP}. } \gre{We also note that  the} computational complexity of using \gre{the Lipschitz inequalities} is ${O}(n^2)$ which is \reb{the same cost as (exact) inference in the GP (using matrix factorization updates)}.

\ylow{We can use the Lipschitz bounds to restrict the limits of the unknown function value for computing the improvement. The upper bound \reb{$U_f$} will always be $f^{u}(x)$, while the lower bound \reb{$L_f$} will depend on the relative value of $\fbest$. In particular, we have the following two cases:}
\[
\reb{L_f} =  
\begin{cases}
\fbest,  & \text{if } \fbest \in \left(f^{l}(x),f^{u}(x) \right) \\
f^{u}(x), & \text{if } \fbest \in \left(f^{u}(x),\infty \right) 
\end{cases}.
\]
\ylow{The second case represents points that cannot improve over the current best value (that are ``rejected'' by the Lipschitz inequalities).}\\
\textbf{Truncated-PI:} We can define a similar variant for the PI acquisition function as \footnote{\reb{Note that the only difference between the usual PI/EI and the truncated version is changing the integral limits to $(L_f,U_f)$ instead of $(-\infty, \infty$).}}:
\begin{align}
TPI(x) = \Phi \left( z(x, \reb{L_f}) \right) - \Phi \left( z(x, \reb{U_f}) \right).
\label{eq:TPI-def}
\end{align}
\textbf{Truncated-EI:} Using the above bounds, the truncated expected improvement for point $x$ is given by:
\begin{align}
 TEI(x) & = - \sigma(x) \cdot z(x,\fbest) \left[\Phi(z(x,\reb{L_f})) - \Phi(z(x,\reb{U_f}) \right] \nonumber \\ 
& + \sigma(x) \cdot \left[ \phi(z(x,\reb{L_f}) - \phi(z(x,\reb{U_f}) \right].
\label{eq:TEI-def}
\end{align}

\mikecommenta {Note that \gre{removing the Lipschitz bounds corresponds to using} $f^{l}(x) = -\infty$ and $f^{u}(x) = \infty$, \gre{and in this case} we recover the usual PI and EI methods in \gre{E}quations~\eqref{eq:PI-def} and~\eqref{eq:EI-def}, respectively.}\\
\begin{figure*}
	\centering
	\includegraphics[width=0.95\textwidth, height=0.35\textwidth]{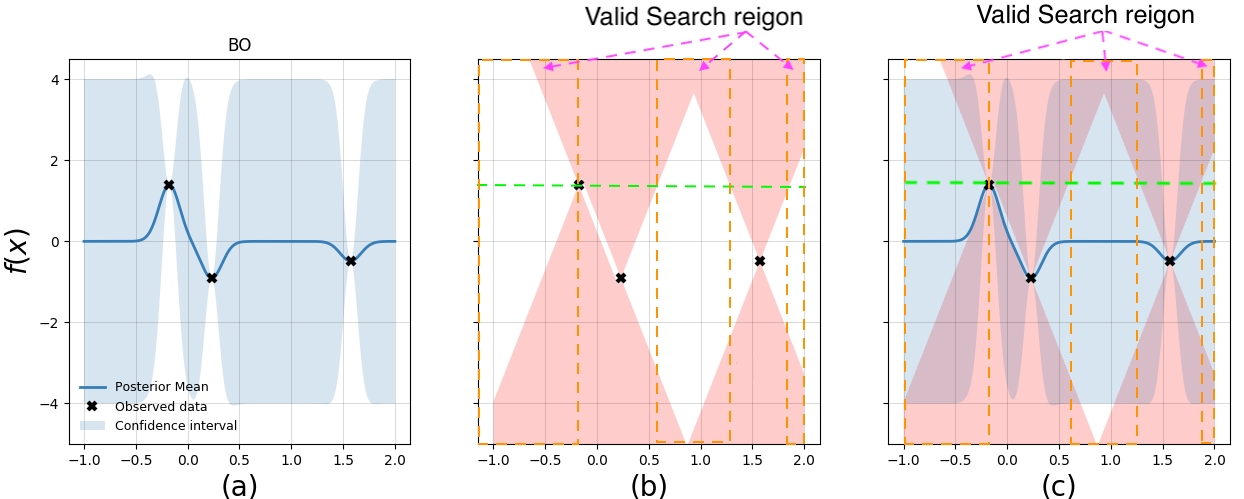}
	\caption{\gre{Visualization of the effect of incorporating the Lipschitz bounds to BO. a) Shows the posterior mean and confidence interval of the conventional BO. b) The red color represents the regions  of the space that are excluded by the Lipschitz bounds. c) Shows the effect of LBO.
		The Grey color represents the uncertainty. Using LBO helps cuts off regions where the posterior variance is high, 
		which prevents over-exploration in unnecessary parts of the space.}  }
	\label{fig:visualization}
\end{figure*}
\textbf{Truncated-UCB:} \ylow{The same strategy can be applied to UCB as follows:}
\begin{align}
TUCB(x) = \min \left\{ \mu(x) + \beta_{t}^{1/2} \sigma(x), f^{u}(x) \right\}.
\label{eq:TPI-def}
\end{align}
\textbf{Accept-Reject:} 
An alternative strategy to incorporate the Lipschitz bounds is to use an accept-reject based mixed acquisition function.
This approach uses the Lipschitz bounds as a sanity-check to accept or reject the value provided by the original acquisition function, similar to LO methods. Formally, if $g(x)$ is the value of the original acquisition function (e.g. $g(x) = UCB(x)$ or $g(x) = \tilde{f}(x)$ for TS), then the mixed acquisition function $\bar{g}(x)$ is given as follows:
\[
\bar{g}(x) =  
\begin{cases}
g(x), & \text{if } g(x) \in [f^{l}(x), f^{u}(x)] \text{ (Accept)} \\
-\infty,  & \text{othewise } \text{ (Reject)}\\
\end{cases}.
\]
We refer to the accept-reject based mixed acquisition functions as AR-UCB and AR-TS, respectively. Note that the accept-reject method is quite generic and can be used with any acquisition function that has values on the same scale as that of the function. \gre{When using an estimate of $L$ it is possible that a good point could be rejected because the estimate of $L$ is too small, but using a growing estimate ensures that such points can again be selected on later iterations.}

\subsection{Regret bound for AR-UCB}
\label{sec:LBO-theory}
In this section, we show that under reasonable assumptions, AR-UCB is provably ``harmless'', in the sense that it retains the good theoretical properties of GP-UCB. We prove the following theorem under the following assumptions:
\begin{itemize}
\item[1] The GP is correctly specified and with infinite observations, the posterior distribution will collapse to the ``true'' function $f$. 
\item[2] The noise in the observations $\sigma$ is small enough for the Lipschitz bounds in Equations~\ref{eq:lip-bounds} to hold.
\item[3] The Lipschitz constant $L$ is known or has been over-estimated using the techniques described in Section~\ref{sec:L-estimation}. 
\end{itemize}
Assumption $1$ is a common assumption made for providing theoretical results for GP-UCB~\citep{srinivas2009gaussian}. Under these assumptions, we obtain the following theorem (proved in Appendix~\ref{app:regret}):
\begin{theorem}
Let $\cD$ be a finite decision space and $\sigma$ be the standard deviation of the noise in the observations. Let $\pi_t$ be a positive scalar such that $\sum_{t} \pi_{t}^{-1} = 1$ and $\delta \in (0,1)$. If we use the AR-UCB algorithm with $\bt = 2 \log(\vert \cD \vert \pi_{t}/\delta)$ assuming that the above conditions $1$-$3$ hold, then the expected cumulative regret $R(T)$ can be bounded as follows:
\begin{align*}
R(T) & \leq \left( 8 / \log(1 + \sigma^{-2}) \right) \beta_{T} \gamma_{T} \sqrt{T}.
\end{align*}
Here, $\gamma_{T}$ refers to the information gain for the selected points and depends on the kernel being used. For the squared exponential kernel, we obtain the following specific bound:
\begin{align*}
R(T) & \leq \left( 8 / \log(1 + \sigma^{-2}) \right) \beta_{T} (\log(T))^{d+1} \sqrt{T}.
\end{align*}
\end{theorem}
The $\gamma_{T}$ term can also be bounded for the Mat\'{e}rn kernel following~\citet{srinivas2009gaussian}. The above theorem shows that under reasonable assumptions, using the Lipschitz bounds in conjunction with GP-UCB cannot result in worse regret. We empirically show that if $L$ is over-estimated, then AR-UCB matches the performance of GP-UCB  in the worst case. 

Note that the above theorem assumes that the GP is correctly specified with the correct hyper-parameters. It also assumes that we are able to specify the correct value of the trade-off parameter $\bt$. These assumptions are not guaranteed to hold in practice and this may result in worse performance of the GP-UCB algorithm. In such cases, our experiments show that using the Lipschitz bounds can lead to better empirical performance than the original GP-UCB.
\section{Experiments}
\label{sec:experiments}

\textbf{Datasets:} We perform an extensive experimental evaluation and present results on twelve synthetic datasets and three real-world tasks. For the synthetic experiments, we use the standard global-optimization benchmarks namely the Branin, Camel, Goldstein Price, Hartmann (2 variants), Michalwicz (3 variants) and  Rosenbrock (4 variants). The closed form and domain for each of these functions is given in~\citet{jamil2013literature}. As examples of real-world tasks, we consider tuning the parameters for a \emph{robot-pushing} simulation (2 variants)~\citep{wang2017max} and tuning the hyper-parameters for logistic regression~\citep{wu2017bayesian}. For the robot pushing example, our aim is to find a good pre-image~\citep{kaelbling2017pre} in order for the robot to push the object to a pre-specified goal location. We follow the experimental protocol from~\citet{wang2017max} and use the negative of the distance to the goal location as the black-box function to maximize. We consider tuning the robot position $r_x, r_y$, and duration of the push $t_{r}$ for the 3D case. We also tune the angle of the push $\theta_{r}$ to make it a 4 dimensional problem. For the hyper-parameter tuning task, we consider tuning the strength of the $\ell_{2}$ regularization (in the range  $[10^{-7},0.9]$), the learning rate for stochastic gradient descent (in the range $[10^{-7},0.05]$)\reb{,} and the number of passes over the data (in the range $[2,15]$). The black-box function is the negative loss on the test set (using a train/test split of $80\%/20\%$) for the MNIST dataset. \\
\textbf{Experimental Setup:} For Bayesian optimization, we use a Gaussian Process prior with the Mat\'{e}rn kernel (with a different length scale for each dimension). We modified the publically available BO package \emph{pybo} of~\citet{hoffmanmodular} to construct the mixed acquisition functions. All the prior hyper-parameters were set and updated across iterations according to the open-source Spearmint package\footnote{https://github.com/hips/spearmint}.
In order to make the optimization invariant to the scale of the function values, similar to Spearmint, we standardize the function values; after each iteration, we centre the observed function values by subtracting their mean and dividing by their standard deviation. We then fit a GP to these rescaled function values and correct for our Lipschitz constant estimate by dividing it by the standard deviation. 
We use DIRECT~\citep{jones1993lipschitzian} in order to optimize the acquisition function in each iteration. \blu{This is one of the standard choices in current works on BO~\citep{eric2008active, martinez2007active, mahendran2012adaptive}, but we expect that Lipschitz information could improve the performance under other choices of the acquisition function optimization approach such as discretization~\citep{snoek2012practical}, adaptive grids~\citep{bardenet2010surrogating}, and other gradient-based methods~\citep{hutter2011sequential, lizotte2012experimental}.
}
\blu{In order to ensure that Bayesian optimization does not get stuck in sub-optimal maxima (either because of the auxiliary optimization or a ``bad'' set of hyper-parameters), on every fourth iteration of BO (or LBO) we choose a random point to evaluate rather than optimizing the acquisition function. This makes the optimization procedure ``harmless'' in the sense that BO (or LBO) will not perform worse than random search~\citep{ahmed2016we}. This has become common in recent  BO methods such as~\citet{bull2011convergence, hutter2011sequential}; and  \citet{falkner2017combining}, and to make the comparison fair we add this ``exploration'' step to all methods.
 Note that in the case of LBO we may need to reject random points until we find one satisfying the Lipschitz inequalities (this does not require evaluating the function).
}
In practice, we found that both the standardization and iterations of random exploration are essential for good performance.\footnote{Note that we verified that our baseline version of BO performs better than or equal to Spearmint across benchmark problems.} All our results are averaged over $10$ independent runs, and each of our figures plots the mean and standard deviation of the absolute error (compared to the global optimum) versus the number of function evaluations. For functions evaluated on log scale, we show the 10\textsuperscript{th} and 90\textsuperscript{th} quantiles.\\ 
\textbf{Algorithms compared:} We compare the performance of Random search, BO, and LBO methods (using both estimated and \emph{True} \reb{Lipschitz constant} $L$) for the EI, PI, UCB and TS acquisition functions. The \emph{True} $L$ was estimated offline using a large number of random points.
For UCB, we set the trade-off parameter $\beta$ according to~\citet{kandasamy2017asynchronous}. For EI and PI, we use Lipschitz bounds to truncate the range of function values for calculating the improvement and use the LBO variants TEI and TPI respectively. For UCB and TS, we use the accept-reject strategy and evaluate the LBO variants AR-UCB and AR-TS respectively. In addition to these, we use random exploration 
as another baseline. We chose the hyper-parameter $\kappa$ (that controls the extent of over-estimating the Lipschitz constant) on the Rosenbrock-4D function  and use the best value of $\kappa$ for all the other datasets and acquisition functions for both BO and LBO. In particular, we set  $\kappa =10$. \\

\textbf{Results:} To make the results easier to read, we divide the results into the following groups:
\begin{enumerate}
	\item 
	LBO provides huge improvements over BO shown in Figure~\ref{fig:kill}. Overall, this represents $21\%$ of all the  test cases.
	\item 
	LBO provides improvements over BO shown in Figure~\ref{fig:improve}. Overall, this represents $9\%$ of all the  test cases.
	\item 
	LBO performs similar to BO shown in 
	 \ref{fig:robot-4D-UCB}. Overall, this represents $60\%$ of all the  test cases.
	\item 
	LBO performs slightly worse than BO shown in Figure~\ref{fig:rosenbrock-4D-PI}. Overall, this represents $10\%$ of all the  test cases.
	\end{enumerate}

\begin{figure*}[!ht]
	
	\centering
	\subfigure[Michalwicz 5D-TS]
	{
		\includegraphics[width=0.31\textwidth,height=0.31\textwidth]{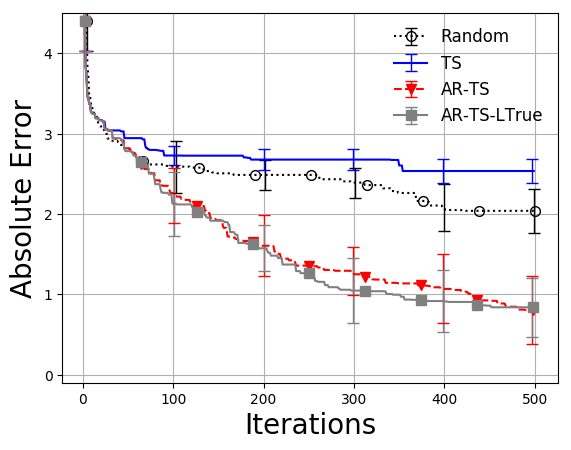}
		\label{fig:michalewicz-5D-TS}
	}
	\subfigure[Rosenbrock 3D-TS]
	{
		\includegraphics[width=0.31\textwidth, height=0.31\textwidth]{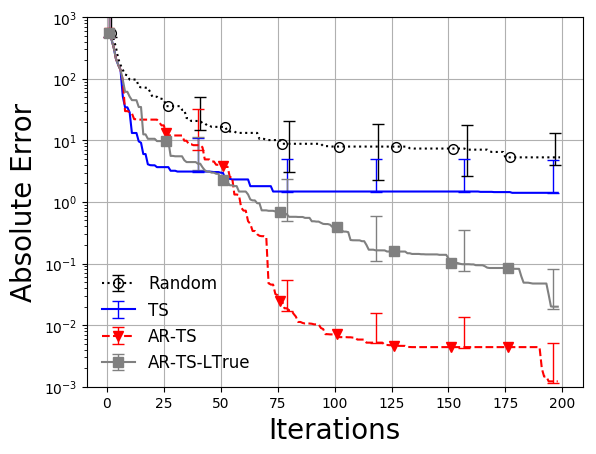}
		\label{fig:rosenbrock-3D-TS}
	}
	\subfigure[Robot pushing 3D-TS]
	{
		\includegraphics[width=0.31\textwidth,height=0.31\textwidth]{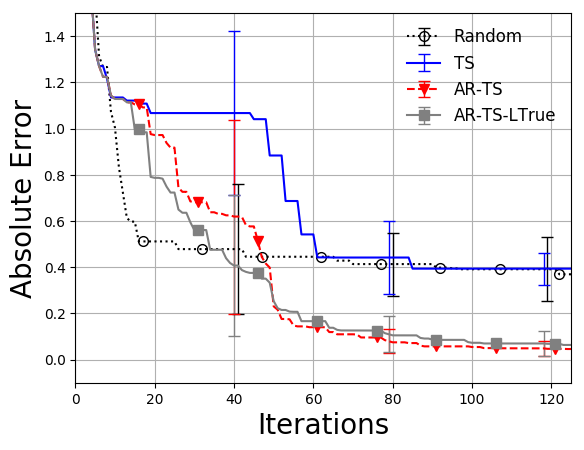}
		\label{fig:robot-3D-TS}
	}	
	\\
	\subfigure[Goldstein 2D-EI]
	{
		\includegraphics[width=0.31\textwidth,height=0.31\textwidth]{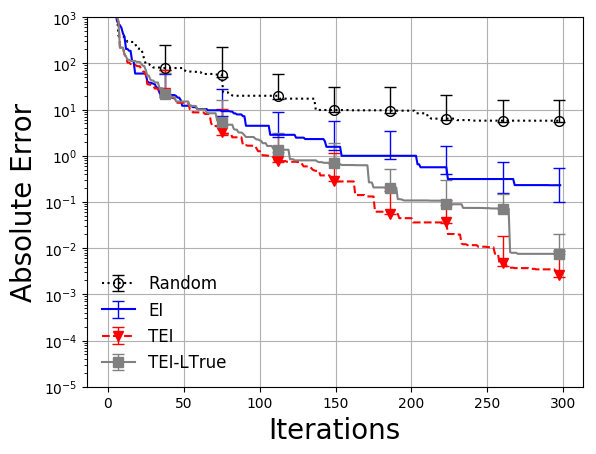}
		\label{fig:rosenbrock-2D-PI}
	}
	\subfigure[Hartmann 3D-EI]
	{
		\includegraphics[width=0.31\textwidth,height=0.31\textwidth]{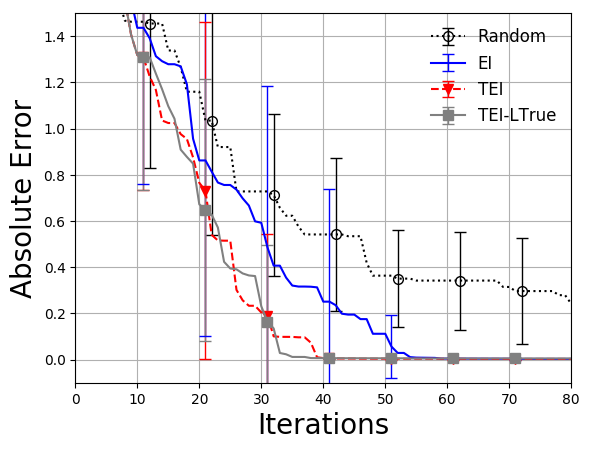}
		\label{fig:hartmann-3D-EI}
	}
	\subfigure[Rosenbrock 5D-UCB]
	{
		\includegraphics[width=0.31\textwidth,height=0.31\textwidth]{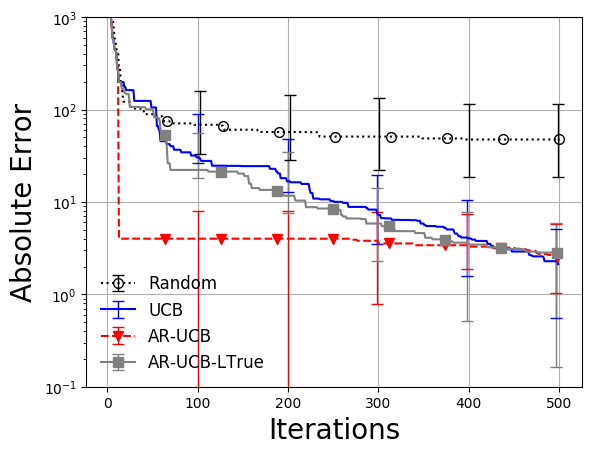}
		\label{fig:rosenbrock-5D-UCB}
	}
	
	\caption{ Examples of functions where LBO provides huge improvements over BO for the different acquisition functions. The figure also shows the performance of random search and LBO using the \emph{True} Lipschitz constant.
	}
	\label{fig:kill}
\end{figure*}
\begin{figure*}[!ht]
	
	\centering
	\subfigure[Rosenbrock 2D-UCB]
	{
		\includegraphics[width=0.31\textwidth,height=0.31\textwidth]{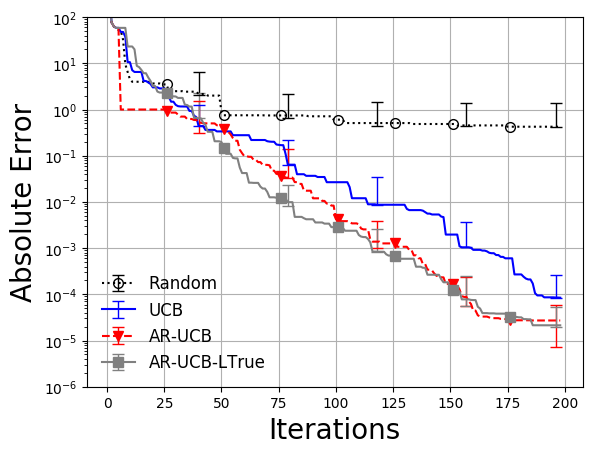}
		\label{fig:improve}
	}
	\subfigure[Robot pushing 4D-UCB]
	{
		\includegraphics[width=0.31\textwidth,height=0.31\textwidth]{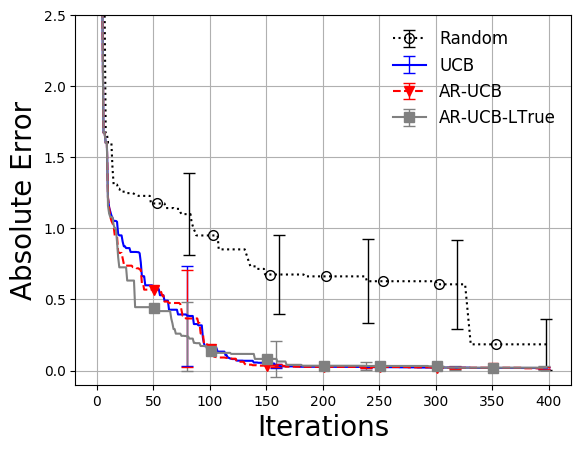}
		\label{fig:robot-4D-UCB}
	}	
	\subfigure[Rosenbrock 4D-PI]
	{
		\includegraphics[width=0.31\textwidth,height=0.31\textwidth]{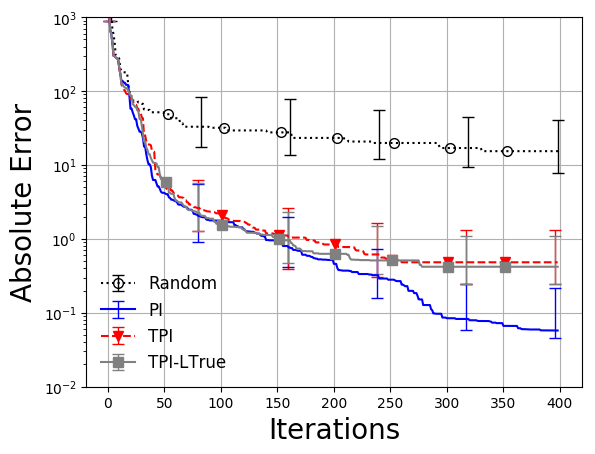}
		\label{fig:rosenbrock-4D-PI}
	}
	
	\caption{Examples of functions where LBO provides some improvement over BO (case a),  LBO performs similar to BO (case b), and  BO  performs slightly better than LBO  (case c).}
	\label{fig:same}
\end{figure*}

\begin{figure*}[!ht]
	\centering
	
	\subfigure[Michalwicz 5D]
	{
		\includegraphics[width=0.31\textwidth,height=0.31\textwidth]{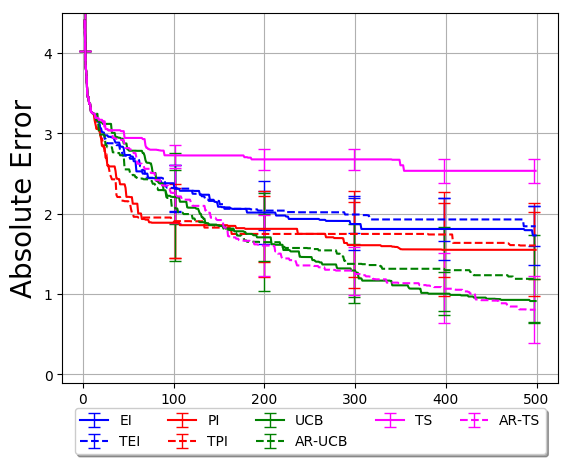}
		\label{fig:joint-mich-5D}
	}
	\subfigure[Rosenbrock 2D]
	{
		\includegraphics[width=0.31\textwidth,height=0.31\textwidth]{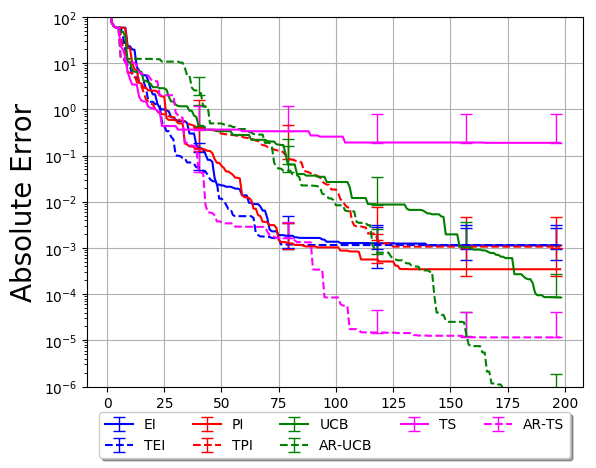}
		\label{fig:joint-rosen-2D}
	}
	\subfigure[Camel 2D-UCB]
	{
		\includegraphics[width=0.31\textwidth,height=0.31\textwidth]{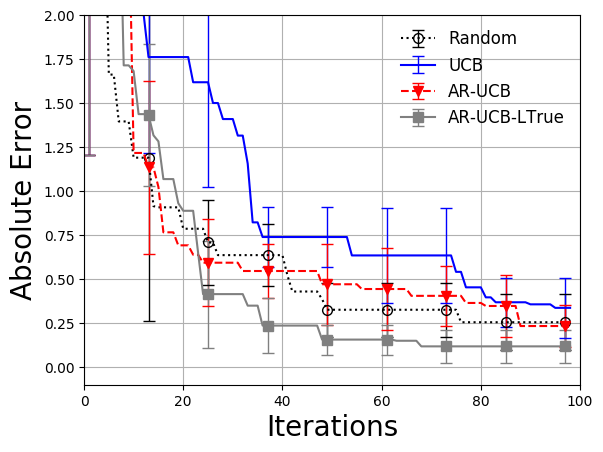}
		\label{fig:neg_camel-2-UCB_large_ucb}	
	}
	
	\caption{(a, b) Examples of functions where LBO boosts the performance of BO with TS (Better seen in color). (c) Example where LBO outperforms BO with UCB when the $\beta$ parameter is too large ($\beta = 10^{16}$). }
	\label{joint}
\end{figure*}

A comparison of the performance across different acquisition functions (for both BO and LBO) on some of the  functions is shown in Figure~\ref{joint}, where we also show an example of UCB where $\beta$ is misspecified. The plots for all functions and methods are available in Appendix C. 
 From these experiments, we can observe:

%

\begin{itemize}
	\item LBO can potentially lead to large gains in performance across acquisition functions and datasets, particularly for TS.
	
	\item Across datasets, we observe that the gains for EI are relatively small, they are occasionally large for PI and UCB and tend to be consistently large for TS. This can be explained as follows: using EI results in under-exploration of the search space, a fact that has been consistently observed and even theoretically proven by~\citet{qin2017improving}. As a result of this, BO does not tend to explore ``bad'' regions when using EI which results in smaller gains from LBO (on the other hand, it may under-explore).
	
	\item TS suffers from exactly the opposite problem: it results in high variance leading to over-exploration of the search space and poor performance. This can be observed in Figures~\ref{fig:michalewicz-5D-TS},~\ref{fig:rosenbrock-3D-TS} and~\ref{fig:robot-3D-TS} where the performance of TS is near random. This has also been observed and noted by~\citet{shahriari2016taking}. For the discrete multi-armed bandit case,~\citet{chapelle2011empirical} multiply the obtained variance estimate by a small number to discourage over-exploration and show that it leads to better results. LBO offers a more principled way of obtaining this same effect and consequently results in making TS more competitive with the other acquisition functions. 
	
	\item The only functions where LBO slightly hurts are Rosenbrock-4D and Goldstein with UCB and PI.
	
	
	\item For Michalwicz-5D (Figure~\ref{fig:joint-mich-5D}), we see that there is no gain for EI, PI, or UCB. However, the gain is huge for TS functions. In fact, even though TS is the worst performing acquisition function on this dataset, its LBO variant AR-TS gives the best performance across all methods. This demonstrates the possible gain that can be obtained from using  mixed acquisition functions. 
	\item We observe a similar trend in Figures~\ref{fig:joint-rosen-2D} 
	where LBO improves TS from near-random performance to being competitive with the best performing methods (while it does not adversely affect the methods performing well). 
	
	\item For the cases where BO performs slightly better than LBO, we notice that the True estimate of $L$ provides compararble performance to BO, so the problem can be narrowed down to finding a good estimate of $L$.
	
	\item \ylow{Figure~\ref{fig:neg_camel-2-UCB_large_ucb} shows examples where LBO saves BO with UCB when the parameter $\beta$ is chosen too large ($\beta = 10^{16}$). In this case BO performs near random, but using LBO leads to better performance than random search.}
\end{itemize}

In any case, our experiments indicate that LBO methods rarely hurt the performance of the original acquisition function. Since they have minimal computational or memory requirements and are simple to implement, these experiments support using  the Lipschitz bounds.
\section{Related work}
\label{sec:related-work}
The Lipschitz condition has been used with BO under different contexts in two previous works~\citep{gonzalez2016batch, sui2015safe}. The aim of~\citet{sui2015safe} is to design a ``safe'' BO algorithm. They assume knowledge of the true Lipschitz constant and exploit Lipschitz continuity to construct a safety threshold in order \gre{to construct a ``safe''  region of the parameter space}.
\gre{This  is different than our goal of improving the performance of existing BO methods, and also different in that we estimate the Lipschitz constant as we run the algorithm.}
 On the other hand,~\citet{gonzalez2016batch} used Lipschitz continuity to model interactions between a batch of points chosen simultaneously in every iteration of BO (referred to as ``Batch'' Bayesian optimization). \gre{This contrasts with our work where we are aiming to improve the performance of existing sequential algorithms (it is possible that our ideas could be used in their framework).}

\vspace*{-2.5ex}
\section{Discussion}
\label{sec:discussion}
In this paper, we have proposed simple ways to combine Lipschitz inequalities with some of the most common BO methods. Our experiments show that this often gives a performance gain, and in the worst case it 
performs similar
to a standard BO method. 
Although we have focused on four of the simplest acquisition functions, it seems that these inequalities could be used within other acquisition functions. For example, 
information-theoretic acquisition functions such as entropy search and their recent extensions rely on sampling a function from the GP and hence the techniques we used for Thompson sampling can be used. We leave a systematic study of these information-theoretic acquisition functions to future study. Further, we expect that the Lipschitz inequalities could also be used in other settings like BO with constraints~\citep{gelbart2014bayesian,hernandez2016general,gardner2014bayesian}, BO methods based on other model classes like neural networks~\citep{snoek2015scalable} or random forests~\citep{hutter2011sequential}, and methods that evaluate more than one $x_t$ at a time~\citep{ginsbourger2010kriging, wang2016parallel}.
Finally, there has been recent interest in first-order Bayesian optimization methods 
~\citep{ahmed2016we,wu2017bayesian}.  If the gradient is Lipschitz continuous then it is possible to use the descent lemma~\citep{bertsekas1999} to obtain Lipschitz bounds that depend on both function values and gradients.


\clearpage
\bibliographystyle{spbasic}
\bibliography{ref}

\begin{thebibliography}{51}
\providecommand{\natexlab}[1]{#1}
\providecommand{\url}[1]{{#1}}
\providecommand{\urlprefix}{URL }
\expandafter\ifx\csname urlstyle\endcsname\relax
  \providecommand{\doi}[1]{DOI~\discretionary{}{}{}#1}\else
  \providecommand{\doi}{DOI~\discretionary{}{}{}\begingroup
  \urlstyle{rm}\Url}\fi
\providecommand{\eprint}[2][]{\url{#2}}

\bibitem[{Ahmed et~al(2016)Ahmed, Shahriari, and Schmidt}]{ahmed2016we}
Ahmed MO, Shahriari B, Schmidt M (2016) Do we need “harmless” {B}ayesian
  optimization and “first-order” {B}ayesian optimization? NIPS Workshop on
  Bayesian Optimization

\bibitem[{Bardenet and K{\'e}gl(2010)}]{bardenet2010surrogating}
Bardenet R, K{\'e}gl B (2010) Surrogating the surrogate: accelerating
  gaussian-process-based global optimization with a mixture cross-entropy
  algorithm. In: International Conference on Machine Learning (ICML),
  Omnipress, pp 55--62

\bibitem[{Bertsekas(2016)}]{bertsekas1999}
Bertsekas DP (2016) Nonlinear Programming, 3rd edn. MIT

\bibitem[{Bull(2011)}]{bull2011convergence}
Bull AD (2011) Convergence rates of efficient global optimization algorithms.
  Journal of Machine Learning Research 12(Oct):2879--2904

\bibitem[{Bunin and Fran{\c{c}}ois(2016)}]{bunin2016lipschitz}
Bunin GA, Fran{\c{c}}ois G (2016) Lipschitz constants in experimental
  optimization. arXiv preprint arXiv:160307847

\bibitem[{Chapelle and Li(2011)}]{chapelle2011empirical}
Chapelle O, Li L (2011) An empirical evaluation of thompson sampling. In:
  Advances in {N}eural {I}nformation {P}rocessing {S}ystems (NIPS), pp
  2249--2257

\bibitem[{Eric et~al(2008)Eric, Freitas, and Ghosh}]{eric2008active}
Eric B, Freitas ND, Ghosh A (2008) Active preference learning with discrete
  choice data. In: Advances in {N}eural {I}nformation {P}rocessing {S}ystems
  (NIPS), pp 409--416

\bibitem[{Falkner et~al(2017)Falkner, Klein, and Hutter}]{falkner2017combining}
Falkner S, Klein A, Hutter F (2017) Combining hyperband and bayesian
  optimization. In: NIPS Workshop on Bayesian Optimization

\bibitem[{Gardner et~al(2014)Gardner, Kusner, Xu, Weinberger, and
  Cunningham}]{gardner2014bayesian}
Gardner JR, Kusner MJ, Xu ZE, Weinberger KQ, Cunningham JP (2014) {B}ayesian
  optimization with inequality constraints. In: International Conference on
  Machine Learning (ICML), pp 937--945

\bibitem[{Gelbart et~al(2014)Gelbart, Snoek, and Adams}]{gelbart2014bayesian}
Gelbart MA, Snoek J, Adams RP (2014) {B}ayesian optimization with unknown
  constraints. arXiv preprint arXiv:14035607

\bibitem[{Ginsbourger et~al(2010)Ginsbourger, Le~Riche, and
  Carraro}]{ginsbourger2010kriging}
Ginsbourger D, Le~Riche R, Carraro L (2010) Kriging is well-suited to
  parallelize optimization. In: Computational intelligence in expensive
  optimization problems, Springer, pp 131--162

\bibitem[{Golovin et~al(2017)Golovin, Solnik, Moitra, Kochanski, Karro, and
  Sculley}]{golovin2017google}
Golovin D, Solnik B, Moitra S, Kochanski G, Karro J, Sculley D (2017) Google
  vizier: A service for black-box optimization. In: Proceedings of the 23rd ACM
  SIGKDD International Conference on Knowledge Discovery and Data Mining, ACM,
  pp 1487--1495

\bibitem[{Gonz{\'a}lez et~al(2016)Gonz{\'a}lez, Dai, Hennig, and
  Lawrence}]{gonzalez2016batch}
Gonz{\'a}lez J, Dai Z, Hennig P, Lawrence N (2016) Batch {B}ayesian
  optimization via local penalization. In: International Conference on
  Artificial Intelligence and Statistics (AISTATS), pp 648--657

\bibitem[{Hendrix et~al(2010)Hendrix, Bogl{\'a}rka
  et~al}]{hendrix2010introduction}
Hendrix EM, Bogl{\'a}rka G, et~al (2010) Introduction to nonlinear and global
  optimization. Springer

\bibitem[{Hennig and Schuler(2012)}]{hennig2012entropy}
Hennig P, Schuler CJ (2012) Entropy search for information-efficient global
  optimization. Journal of Machine Learning Research 13(Jun):1809--1837

\bibitem[{Hern{\'a}ndez-Lobato et~al(2014)Hern{\'a}ndez-Lobato, Hoffman, and
  Ghahramani}]{hernandez2014predictive}
Hern{\'a}ndez-Lobato JM, Hoffman MW, Ghahramani Z (2014) Predictive entropy
  search for efficient global optimization of black-box functions. In: Advances
  in {N}eural {I}nformation {P}rocessing {S}ystems (NIPS), pp 918--926

\bibitem[{Hern{\'a}ndez-Lobato et~al(2016)Hern{\'a}ndez-Lobato, Gelbart, Adams,
  Hoffman, and Ghahramani}]{hernandez2016general}
Hern{\'a}ndez-Lobato JM, Gelbart MA, Adams RP, Hoffman MW, Ghahramani Z (2016)
  A general framework for constrained {B}ayesian optimization using
  information-based search. Journal of Machine Learning Research
  17(1):5549--5601

\bibitem[{Hoffman and Shahriari(2014)}]{hoffmanmodular}
Hoffman MW, Shahriari B (2014) Modular mechanisms for {B}ayesian optimization.
  In: NIPS Workshop on Bayesian Optimization, pp 1--5

\bibitem[{Hutter et~al(2011)Hutter, Hoos, and
  Leyton-Brown}]{hutter2011sequential}
Hutter F, Hoos HH, Leyton-Brown K (2011) Sequential model-based optimization
  for general algorithm configuration. In: International Conference on Learning
  and Intelligent Optimization, Springer, pp 507--523

\bibitem[{Jamil and Yang(2013)}]{jamil2013literature}
Jamil M, Yang XS (2013) A literature survey of benchmark functions for global
  optimisation problems. International Journal of Mathematical Modelling and
  Numerical Optimisation 4(2):150--194

\bibitem[{Jones et~al(1993)Jones, Perttunen, and
  Stuckman}]{jones1993lipschitzian}
Jones DR, Perttunen CD, Stuckman BE (1993) Lipschitzian optimization without
  the lipschitz constant. Journal of Optimization Theory and Applications
  79(1):157--181

\bibitem[{Jones et~al(1998)Jones, Schonlau, and Welch}]{jones1998efficient}
Jones DR, Schonlau M, Welch WJ (1998) Efficient global optimization of
  expensive black-box functions. Journal of Global optimization 13(4):455--492

\bibitem[{Kaelbling and Lozano-P{\'e}rez(2017)}]{kaelbling2017pre}
Kaelbling LP, Lozano-P{\'e}rez T (2017) Pre-image backchaining in belief space
  for mobile manipulation. In: Robotics Research, Springer, pp 383--400

\bibitem[{Kandasamy et~al(2017)Kandasamy, Krishnamurthy, Schneider, and
  Poczos}]{kandasamy2017asynchronous}
Kandasamy K, Krishnamurthy A, Schneider J, Poczos B (2017) Asynchronous
  parallel {B}ayesian optimisation via thompson sampling. arXiv preprint
  arXiv:170509236

\bibitem[{Kim and Choi(2019)}]{kim2019local}
Kim J, Choi S (2019) On local optimizers of acquisition functions in bayesian
  optimization. arXiv preprint arXiv:190108350

\bibitem[{Kushner(1964)}]{kushner1964new}
Kushner HJ (1964) A new method of locating the maximum point of an arbitrary
  multipeak curve in the presence of noise. Journal of Basic Engineering
  86(1):97--106

\bibitem[{Li et~al(2016)Li, Jamieson, DeSalvo, Rostamizadeh, and
  Talwalkar}]{li2016efficient}
Li L, Jamieson K, DeSalvo G, Rostamizadeh A, Talwalkar A (2016) Efficient
  hyperparameter optimization and infinitely many armed bandits. arXiv preprint
  arXiv:160306560

\bibitem[{Lizotte et~al(2012)Lizotte, Greiner, and
  Schuurmans}]{lizotte2012experimental}
Lizotte DJ, Greiner R, Schuurmans D (2012) An experimental methodology for
  response surface optimization methods. Journal of Global Optimization
  53(4):699--736

\bibitem[{Mahendran et~al(2012)Mahendran, Wang, Hamze, and
  De~Freitas}]{mahendran2012adaptive}
Mahendran N, Wang Z, Hamze F, De~Freitas N (2012) Adaptive mcmc with bayesian
  optimization. In: International Conference on Artificial Intelligence and
  Statistics (AISTATS), pp 751--760

\bibitem[{Malherbe and Vayatis(2017)}]{malherbe2017global}
Malherbe C, Vayatis N (2017) Global optimization of lipschitz functions. In:
  International Conference on Machine Learning (ICML), pp 2314--2323,
  \urlprefix\url{http://proceedings.mlr.press/v70/malherbe17a.html}

\bibitem[{Martinez-Cantin et~al(2007)Martinez-Cantin, de~Freitas, Doucet, and
  Castellanos}]{martinez2007active}
Martinez-Cantin R, de~Freitas N, Doucet A, Castellanos JA (2007) Active policy
  learning for robot planning and exploration under uncertainty. In: Robotics:
  Science and Systems, vol~3, pp 321--328

\bibitem[{Mo{\v{c}}kus(1975)}]{movckus1975bayesian}
Mo{\v{c}}kus J (1975) On {B}ayesian methods for seeking the extremum. In:
  Optimization Techniques IFIP Technical Conference, Springer, pp 400--404

\bibitem[{Pint{\'e}r(1996)}]{pinter1991global}
Pint{\'e}r JD (1996) Global optimization in action: continuous and Lipschitz
  optimization: algorithms, implementations and applications, vol~6. Springer
  Science \& Business Media Dordrecht

\bibitem[{Piyavskii(1972)}]{piyavskii1972algorithm}
Piyavskii S (1972) An algorithm for finding the absolute extremum of a
  function. USSR Computational Mathematics and Mathematical Physics
  12(4):57--67

\bibitem[{Qin et~al(2017)Qin, Klabjan, and Russo}]{qin2017improving}
Qin C, Klabjan D, Russo D (2017) Improving the expected improvement algorithm.
  In: Advances in {N}eural {I}nformation {P}rocessing {S}ystems (NIPS), pp
  5387--5397

\bibitem[{Rasmussen and Williams(2006)}]{rasmussen2006gaussian}
Rasmussen CE, Williams CK (2006) Gaussian processes for machine learning. MIT
  Press

\bibitem[{Rios and Sahinidis(2013)}]{rios2013derivative}
Rios LM, Sahinidis NV (2013) Derivative-free optimization: a review of
  algorithms and comparison of software implementations. Journal of Global
  Optimization 56(3):1247--1293

\bibitem[{Shahriari et~al(2014)Shahriari, Wang, Hoffman, Bouchard-C\^ot\'e, and
  de~Freitas}]{shahriari2014entropy}
Shahriari B, Wang Z, Hoffman MW, Bouchard-C\^ot\'e A, de~Freitas N (2014) An
  entropy search portfolio. In: NIPS Workshop on Bayesian Optimization

\bibitem[{Shahriari et~al(2016)Shahriari, Swersky, Wang, Adams, and
  de~Freitas}]{shahriari2016taking}
Shahriari B, Swersky K, Wang Z, Adams RP, de~Freitas N (2016) Taking the human
  out of the loop: A review of {B}ayesian optimization. Proceedings of the IEEE
  104(1):148--175

\bibitem[{Shubert(1972)}]{shubert1972sequential}
Shubert BO (1972) A sequential method seeking the global maximum of a function.
  SIAM Journal on Numerical Analysis 9(3):379--388

\bibitem[{Snoek et~al(2012)Snoek, Larochelle, and Adams}]{snoek2012practical}
Snoek J, Larochelle H, Adams RP (2012) Practical {B}ayesian optimization of
  machine learning algorithms. Advances in {N}eural {I}nformation {P}rocessing
  {S}ystems (NIPS)

\bibitem[{Snoek et~al(2015)Snoek, Rippel, Swersky, Kiros, Satish, Sundaram,
  Patwary, Prabhat, and Adams}]{snoek2015scalable}
Snoek J, Rippel O, Swersky K, Kiros R, Satish N, Sundaram N, Patwary M, Prabhat
  M, Adams R (2015) Scalable {B}ayesian optimization using deep neural
  networks. In: International Conference on Machine Learning (ICML), pp
  2171--2180

\bibitem[{Srinivas et~al(2010)Srinivas, Krause, Kakade, and
  Seeger}]{srinivas2009gaussian}
Srinivas N, Krause A, Kakade SM, Seeger M (2010) Gaussian process optimization
  in the bandit setting: No regret and experimental design. In: International
  Conference on Machine Learning (ICML), pp 1015--1022

\bibitem[{Stein(2012)}]{stein2012interpolation}
Stein ML (2012) Interpolation of spatial data: some theory for kriging.
  Springer Science \& Business Media

\bibitem[{Sui et~al(2015)Sui, Gotovos, Burdick, and Krause}]{sui2015safe}
Sui Y, Gotovos A, Burdick J, Krause A (2015) Safe exploration for optimization
  with gaussian processes. In: International Conference on Machine Learning
  (ICML), pp 997--1005

\bibitem[{Thompson(1933)}]{thompson1933likelihood}
Thompson WR (1933) On the likelihood that one unknown probability exceeds
  another in view of the evidence of two samples. Biometrika 25(3/4):285--294

\bibitem[{Villemonteix et~al(2009)Villemonteix, Vazquez, and
  Walter}]{villemonteix2009informational}
Villemonteix J, Vazquez E, Walter E (2009) An informational approach to the
  global optimization of expensive-to-evaluate functions. Journal of Global
  Optimization 44(4):509

\bibitem[{Wang et~al(2016)Wang, Clark, Liu, and Frazier}]{wang2016parallel}
Wang J, Clark SC, Liu E, Frazier PI (2016) Parallel {B}ayesian global
  optimization of expensive functions. arXiv preprint arXiv:160205149

\bibitem[{Wang and Jegelka(2017)}]{wang2017max}
Wang Z, Jegelka S (2017) Max-value entropy search for efficient bayesian
  optimization. In: International Conference on Machine Learning (ICML)

\bibitem[{Wilson et~al(2018)Wilson, Hutter, and
  Deisenroth}]{wilson2018maximizing}
Wilson J, Hutter F, Deisenroth M (2018) Maximizing acquisition functions for
  bayesian optimization. In: NIPS, pp 9884--9895

\bibitem[{Wu et~al(2017)Wu, Poloczek, Wilson, and Frazier}]{wu2017bayesian}
Wu J, Poloczek M, Wilson AG, Frazier P (2017) Bayesian optimization with
  gradients. In: Advances in {N}eural {I}nformation {P}rocessing {S}ystems
  (NIPS), pp 5267--5278

\end{thebibliography}
 
\clearpage
\onecolumn
\appendix
\section{Proof for Lipschitz constant estimation}
\label{app:proofs}

In this section we analyze the minimum number of iterations required before we can guarantee (in expectation) that we'll have a point $x$ satisfying
\begin{equation}
f(x) - f(x^*) \leq \epsilon,
\label{eq:opt}
\end{equation}
for a given accuracy tolerance $\epsilon$. Here we assume that $x^*$ is a globally-optimal solution (assumed to exist), the domain of $x$ is a hyper-cube $\mathcal{X}$ in $\R^d$, and $f$ is Lipschitz-continuous. We use $L$ as the minimum value we can use for the Lipschitz constant of $f$. We first consider the case of random selection, followed by random selection with pruning based on any upper bound on $L$, and finally random selection with pruning based on a growing estimate of the Lipschitz constant.

\subsection{Random Selection}

Our first result gives a lower bound on the volume of the solution space where the $x$ satisfy~\eqref{eq:opt}.
\begin{lemma}
For a Lipschitz-continuous function $f$ defined on a hyper-cube $\mathcal{X}$, the volume of $\mathcal{X}$ satisfying~\eqref{eq:opt} is $\Omega((\epsilon/L)^d)$.
\end{lemma}
\begin{proof}
By the Lipschitz inequality we have for any solution $x^*$ that
\[
|f(x) - f(x^*)| \leq L\norm{x-x^*},
\]
for any $x \in \mathcal{X}$. Choose some particular solution $x^*$, and let $\mathcal{B}$ be the set of $x$ satisfying $L\norm{x-x^*} \leq \epsilon$. 
Notice that all $x \in \mathcal{B} \cap \mathcal{X}$ satisfy $\eqref{eq:opt}$, so it is sufficient to show that $|\mathcal{B} \cap \mathcal{X}| = \Omega((\epsilon/L)^d)$.

Since $\mathcal{B}$ is the set of points satisfying $\norm{x-x^*} \leq \epsilon/L$, it is a hyper-shere of radius $\epsilon/L$ which means its volume is $\frac{\pi^{d/2}(\epsilon/L)^d}{(d/2)!}$. The case where $\mathcal{B}$ has the smallest intersection with $\mathcal{X}$ is when $x^*$ is at a vertex in the hyper-cube; in this case we have that exactly one orthant of $\mathcal{B}$ intersecting with $\mathcal{X}$. Since there are $2^d$ orthants (of equal size), in the worst case we have $|\mathcal{B}\cap\mathcal{X}| \geq |\mathcal{B}|/2^d = \frac{\pi^{\d/2}(\epsilon/L)^d}{2^d(d/2)!} = \Omega((\epsilon/L)^d)$ (for fixed dimension $d$).
\end{proof}

Next we give a lower bound on the probability that a random iterate $x$ is a point satisfying~\eqref{eq:opt}
\begin{lemma}
For a Lipschitz-continuous function $f$ defined on a hyper-cube $\mathcal{X}$, a point $x$ chosen uniformly at random from $\mathcal{X}$ satisfies~\eqref{eq:opt} with probability $\Omega((\epsilon/L)^d)$.
\end{lemma}
\begin{proof}
The previous lemma shows that there is a volume of size $\Omega((\epsilon/L)^d)$ in $\mathcal{X}$ containing solutions. Thus, the probability that random point in $\mathcal{X}$ is a solution is $\Omega((\epsilon/L)^d/|\mathcal{X}|) = \Omega((\epsilon/L)^d)$ (for a fixed hyper-cube size).
\end{proof}

Finally, we can give an upper bound on the expected number of iterations before we have an $x_t$ satisfying~\eqref{eq:opt}.
\begin{lemma}
For a Lipschitz-continuous function $f$ defined on a hyper-cube $\mathcal{X}$, if we independently sample points $\{x_1, x_2, \dots\}$ uniformly at random from $\mathcal{X}$, then in expectation we find a point $x$ satisfying~\eqref{eq:opt} after $O((L/\epsilon)^d)$ samples.
\end{lemma}
\begin{proof}
From the previous lemma, each independent sample finds a solution with probability $\Omega((\epsilon/L)^d$. Viewing each sample as a Bernoulli trial, the expected number of iterations before we find a solution is a geometric random variable with success probability $\Omega((\epsilon/L)^d$. Since the expectation of a geometric random variable is the inverse of its success probability, in expectation we find a solution after $O((L/\epsilon)^d)$ samples.
\end{proof}

Instead of ``number of samples $t$ to reach accuracy $\epsilon$'', we could equivalently state the result in terms ``expected error at iteration $t$'' (simple regret) by inverting the relationship between $t$ and $\epsilon$. This would give an expected error on iteration $t$ of $O(L/t^{1/d})$.

\subsection{Random Selection, Pruning based on the True Lipschitz Constant}

In the previous section, we considered choosing points $x_t$ uniformly from $\mathcal{X}$. Consider the case where we are given $L$ (or an upper bound on it), and instead sample uniformly from $\mathcal{X}$ intersected with the points that are not ruled out by the Lipschitz inequalities. Note that this restriction cannot rule out points in $\mathcal{B}$ unless we already  have an $\epsilon$-optimal solution, and thus the arguments from the previous section still apply.

\subsection{Random Selection, Pruning based on a Growing Lipschitz Constant Estimate}

Unfortunately, if we use an estimate $\hat{L}_t$ of $L$ instead of an $L$ satisfying the Lipschitz inequality, we could reject an approximate solution. However, if $\hat{L}_t$ grows with $t$ then eventually it is sufficiently large that we will not reject an approximate solution (unless we already have an $\epsilon$-optimal solution). Thus, a crude bound on the expected number of iterations before we find a solution with accuracy $\epsilon$ is given by $O((L/\epsilon)^d + T)$, where $T$ is the first iteration $t$ beyond which we always have $\hat{L} \geq L$. Thus, if we choose the sequence $\hat{L}_t$ such that $T = O((L/\epsilon)^d)$, then LO with an estimated $\hat{L}_t$ is harmless as it requires the same expected number of iterations as random guessing. A simple example of a sequence of $\hat{L}$ values satisfying this property would be to choose $\hat{L}_t = tL(\epsilon/L)^d$, which grows extremely-slowly (for small $\epsilon$ and non-trivial $d$ or $L$). Larger sequences would imply a smaller $T$ and hence also would be harmless.
\section{Regret Bound}
\label{app:regret}
\setcounter{theorem}{0}
\begin{theorem}
Let $\cD$ be a finite decision space and $\sigma$ be the standard deviation of the noise in the observations. Let $\pi_t$ be a positive scalar such that $\sum_{t} \pi_{t}^{-1} = 1$ and $\delta \in (0,1)$. If we use the AR-UCB algorithm with $\bt = 2 \log(\vert \cD \vert \pi_{t}/\delta)$ assuming that the above conditions $1$-$3$ hold, then the expected cumulative regret $R(T)$ can be bounded as follows:
\begin{align*}
R(T) & \leq \left( 8 / \log(1 + \sigma^{-2}) \right) \beta_{T} \gamma_{T} \sqrt{T}.
\end{align*}
Here, $\gamma_{T}$ refers to the information gain for the selected points and depends on the kernel being used. For the squared exponential kernel, we obtain the following specific bound:
\begin{align*}
R(T) & \leq \left( 8 / \log(1 + \sigma^{-2}) \right) \beta_{T} (\log(T))^{d+1} \sqrt{T}
\end{align*}
\end{theorem}
\begin{proof}
\begin{align}
\intertext{By definition of Lipschitz bounds and assuming we know the true Lipschitz constant $L$, at iteration $t$, for all $x$,}
\fl(x) & \leq f(x) \leq \fu(x). & \label{eq:lip-bounds-2} 
\end{align}
We now use the following lemma from~\cite{srinivas2009gaussian}:
\begin{lemma}[Lemma 5.1 in~\cite{srinivas2009gaussian}]
Denoting $\cD$ as a finite decision space, let $\pi_t > 0$ and $\sum_{t} \pi_{t}^{-1} = 1$. Choose $\bt = 2 \log(\vert \cD \vert \pi_{t}/\delta)$ where $\delta \in (0,1)$. Then, for all $x \in \cD$ and $t \geq 1$, with probability $1 - \delta$, 
\begin{align}
\vert f(x) - \mut(x) \vert \leq \bt \sigmat(x). \label{eq:sri-lemma}
\end{align}
\label{lemma:sri-lemma}
\end{lemma} 
\begin{align}
\intertext{From Equations~\ref{eq:lip-bounds-2} and~\ref{eq:sri-lemma},}
f(\xopt) & \leq \min\{\fu(\xopt), \mut(\xopt) + \bt \sigmat(\xopt)  \}. & \label{eq:inter-1} \\
\intertext{For the point $\xt$ selected at round $t$, the following relation holds because of the Accept-Reject condition:}
\fl(\xt) & \leq \mut(\xt) + \bt \sigmat(\xt) \leq \fu(\xt). &\label{eq:acc-rej} \\
\intertext{The following holds because of the definition of the UCB rule:}
\mut(\xt) + \bt \sigmat(\xt) & \geq \mut(\xopt) + \bt \sigmat(\xopt). &\label{eq:ucb-def} 
\intertext{From Equations~\ref{eq:sri-lemma} and~\ref{eq:acc-rej}}
\mut(\xt) + \bt \sigmat(\xt) & \leq \min\{ f(\xt) + 2 \bt \sigmat(\xt), \fu(\xt) \}. & \label{eq:inter-2} 
\end{align}
\begin{align*}
\intertext{Let $r_t$ be the instantaneous regret in round $t$. Then,}
r_t & = f(\xopt) - f(\xt) \\
& \leq \min\{\fu(\xopt), \mut(\xopt) + \bt \sigmat(\xopt)  \} - f(\xt) & \tag{From Equation~\ref{eq:inter-1}} \\
& \leq \min\{\fu(\xopt), \mut(\xt) + \bt \sigmat(\xt)  \} - f(\xt) & \tag{From Equation~\ref{eq:ucb-def}} \\
& = \min\{\fu(\xopt) - f(\xt), \mut(\xt) + \bt \sigmat(\xt) - f(\xt) \} & \tag{$\min\{a,b\} - c = \min\{a-c,b-c\}$} \\
& \leq \mut(\xt) + \bt \sigmat(\xt) - f(\xt) & \tag{$\min\{a,b\} \leq b$} \\
& \leq \min\{ f(\xt) + 2 \bt \sigmat(\xt), \fu(\xt) \} - f(\xt)& \tag{From Equation~\ref{eq:inter-2}} \\
& = \min\{ 2 \bt \sigmat(\xt), \fu(\xt) - f(\xt) \} & \tag{$\min\{a,b\} - c = \min\{a-c,b-c\}$} \\
\implies r_t & \leq \min\{ 2 \bt \sigmat(\xt), \fu(\xt) - \fl(\xt) \}. & \tag{From Equation~\ref{eq:lip-bounds-2}}
\end{align*}
\begin{align*}
\intertext{Let us now consider the term $\fu(\xt) - \fl(\xt)$.}
\fu(\xt) - \fl(\xt) & = \min_{i \in [t-1]} \left\{ f(x_{i}) + L \vert \vert \xt - x_{i} \vert \vert_{2} \right\} - \max_{i \in [t-1]} \left\{ f(x_{i}) - L \vert \vert \xt - x_{i} \vert \vert_{2} \right\} & \tag{By Equation ~\ref{eq:lip-bounds}} \\
& = \min_{i \in [t-1]} \left\{ f(x_{i}) + L \vert \vert \xt - x_{i} \vert \vert_{2} \right\} + \min_{i \in [t-1]} \left\{ - f(x_{i}) + L \vert \vert \xt - x_{i} \vert \vert_{2} \right\} & \tag{$-\max\{a,b\} = \min\{-a,-b\}$} \\
& \leq \min_{i \in [t-1]} \left\{ f(x_{i}) + L \vert \vert \xt - x_{i} \vert \vert_{2} - f(x_{i}) + L \vert \vert \xt - x_{i} \vert \vert_{2} \right\} &\tag{$\min\{a_i + b_i\} \geq \min\{a_i\} + \min\{b_i\}$} \\
\implies \fu(\xt) - \fl(\xt) & \leq 2L \min_{i \in [t-1]} \left\{ \vert \vert \xt - x_{i} \vert \vert_{2} \right\} .
\end{align*}
\begin{align*}
\intertext{From the above equations,}
r_t & \leq \min \left\{ 2 \bt \sigmat(\xt),  2L \min_{i \in [t-1]} \left\{ \vert \vert \xt - x_{i} \vert \vert_{2} \right\} \right\}. \\
\intertext{Let $R(T)$ be the cumulative regret after $T$ rounds.}
R(T) & = \sum_{t = 1}^{T} r_t \leq \sum_{t= 1}^{T} \left[ \min \left\{ 2 \bt \sigmat(\xt),  2L \min_{i \in [t-1]} \left\{ \vert \vert \xt - x_{i} \vert \vert_{2} \right\} \right\} \right] \\
R(T) & \leq \min \left\{ 2 \sum_{t=1}^{T} \bt \sigmat(\xt), 2L \sum_{t=1}^{T} \min_{i \in [t-1]} \left\{ \vert \vert \xt - x_{i} \vert \vert_{2} \right\} \right\} &\tag{$\min\{\sum_i a_i \} \geq \sum_i \min\{a_i\}$} \\
\end{align*}
We now bound the term $2 \sum_{t=1}^{T} \bt \sigmat(\xt)$ using the lemma in~\cite{srinivas2009gaussian} which we restate next:
\begin{lemma}[Lemma $5.4$ in~\cite{srinivas2009gaussian}]
Choosing $\bt = 2 \log(\vert \cD \vert \pi_{t}/\delta)$,
\begin{align*}
2 \sum_{t=1}^{T} \bt \sigmat(\xt) & \leq C_{1} \gamma_{T} \sqrt{T}.
\end{align*}
where $C_{1} = \left( 8 / \log(1 + \sigma^{-2}) \right) \beta_{T}$. Here $\gamma_{T}$ refers to the information gain for the selected points. 
\end{lemma} 
Using the above lemma, we obtain the following bound:
\begin{align*}
R(T) & \leq \min \left\{C_{1} \gamma_{T} \sqrt{T}, 2L \sum_{t=1}^{T} \min_{i \in [t-1]} \left\{ \vert \vert \xt - x_{i} \vert \vert_{2} \right\} \right\} \\
\implies R(T) & \leq \left( 8 / \log(1 + \sigma^{-2}) \right) \beta_{T} \gamma_{T} \sqrt{T}.
\end{align*}
\end{proof}
\section{Additional Experimental Results}
\label{app:exps}
Below we show the results of all the experiments for all the datasets as follows: 

\begin{itemize}
 \item Figure~\ref{fig:all-TS} shows the performance of Random search, BO, and LBO (using both  estimated and \emph{True} $L$) for the TS acquisition function.
 \item Figure~\ref{fig:all-UCB} shows the performance of Random search, BO, and LBO (using both  estimated and \emph{True} $L$) for the UCB acquisition function.
 \item Figure~\ref{fig:all-EI} shows the performance of Random search, BO, and LBO (using both  estimated and \emph{True} $L$) for the EI acquisition function.
 \item Figure~\ref{fig:all-PI} shows the performance of Random search, BO, and LBO (using both  estimated and \emph{True} $L$) for the PI acquisition function.
 \item Figure~\ref{fig:all-joint} shows the performance of BO  and LBO using the estimated  $L$ for the all acquisition function.  
 \item Figure~\ref{fig:all-large-UCB} shows the performance of Random search, BO, and LBO (using both  estimated and \emph{True} $L$) for the UCB acquisition function with very large $\beta=10^{16}$.
\end{itemize}

\begin{figure*}[!ht]
	\centering
	\subfigure[Branin 2D]
	{
		\includegraphics[width=0.31\textwidth,height=0.28\textwidth]{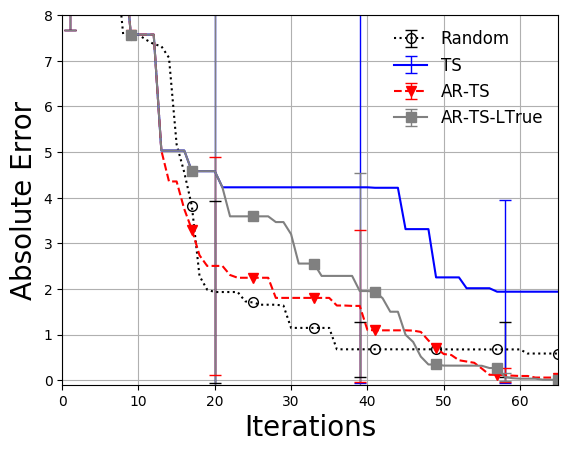}
		\label{fig:branin-2D-TSa}
	}
	\subfigure[Camel 2D]
	{
		\includegraphics[width=0.31\textwidth,height=0.28\textwidth]{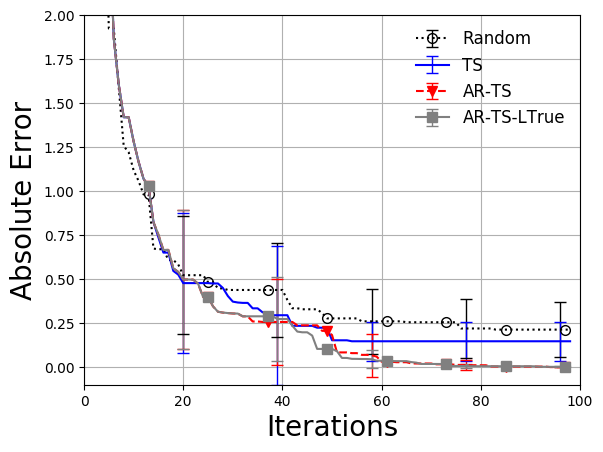}
		\label{fig:camel-2D-TSa}
	}
	\subfigure[Goldstein Price 2D]
	{
		\includegraphics[width=0.31\textwidth,height=0.28\textwidth]{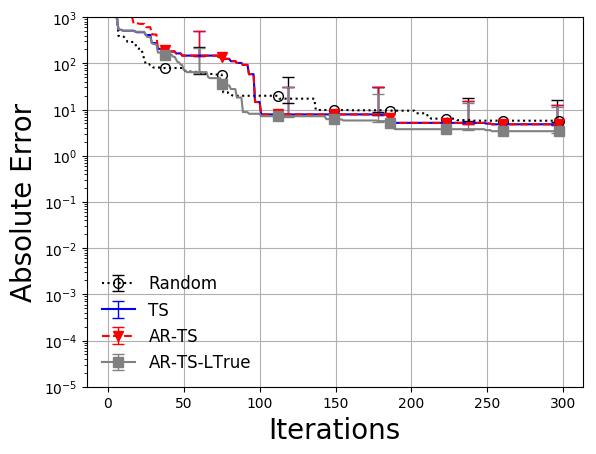}
		\label{fig:goldstein-2D-TSa}
	}
	\\
	\subfigure[Michalwicz 2D]
	{
		\includegraphics[width=0.31\textwidth,height=0.28\textwidth]{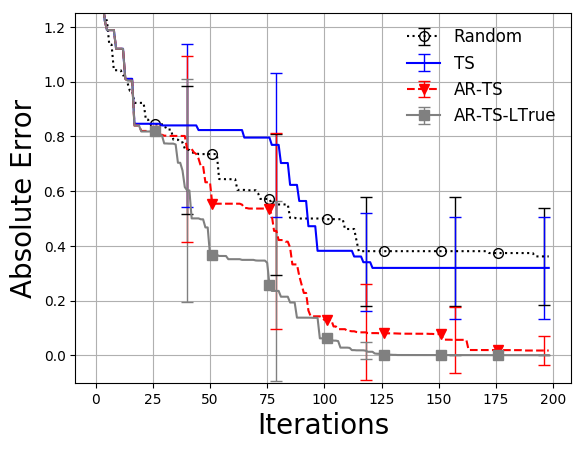}
		\label{fig:mich-2D-TSa}
	}
	\subfigure[Michalwicz 5D]
	{
		\includegraphics[width=0.31\textwidth,height=0.28\textwidth]{neg_michalewicz-5-TS.png}
		\label{fig:mich-5D-TSa}
	}
	\subfigure[Michalwicz 10D]
	{
		\includegraphics[width=0.31\textwidth,height=0.28\textwidth]{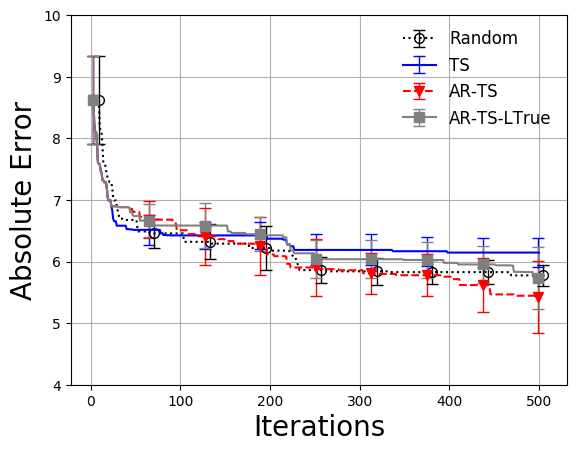}
		\label{fig:mich-10D-TSa}
	}
	\\
	\subfigure[Rosenbrock 2D]
	{
		\includegraphics[width=0.31\textwidth,height=0.28\textwidth]{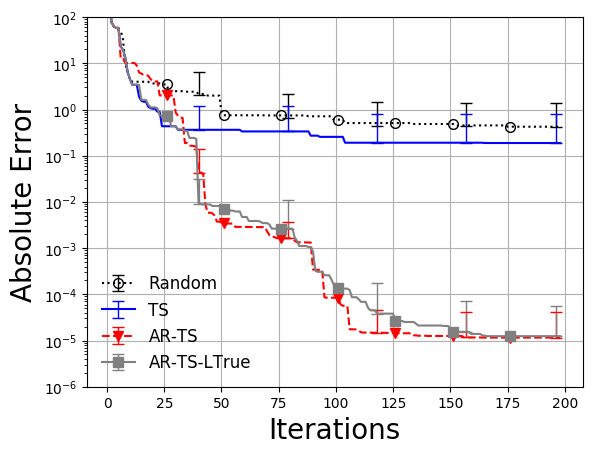}
		\label{fig:rosenbrock-2D-TSa}
	}
	\subfigure[Hartmann 3D]
	{
		\includegraphics[width=0.31\textwidth,height=0.28\textwidth]{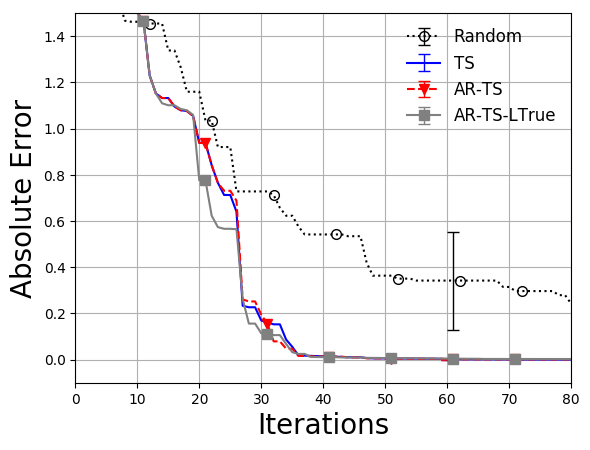}
		\label{fig:hartmann-3D-TSa}
	}
	\subfigure[Hartmann 6D]
	{
		\includegraphics[width=0.31\textwidth,height=0.28\textwidth]{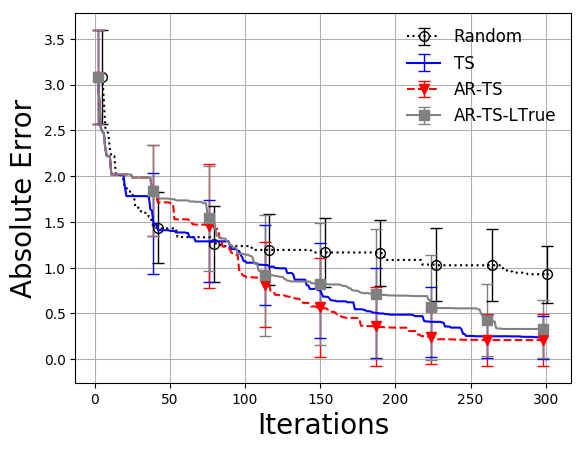}
		\label{fig:hartmann-6D-TSa}
	}
	\\
	\subfigure[Rosenbrock 3D]
	{
		\includegraphics[width=0.31\textwidth,height=0.28\textwidth]{neg_rosen-3-TS.png}
		\label{fig:rosenbrock-3D-TSa}
	}
	\subfigure[Rosenbrock 4D]
	{
		\includegraphics[width=0.31\textwidth,height=0.28\textwidth]{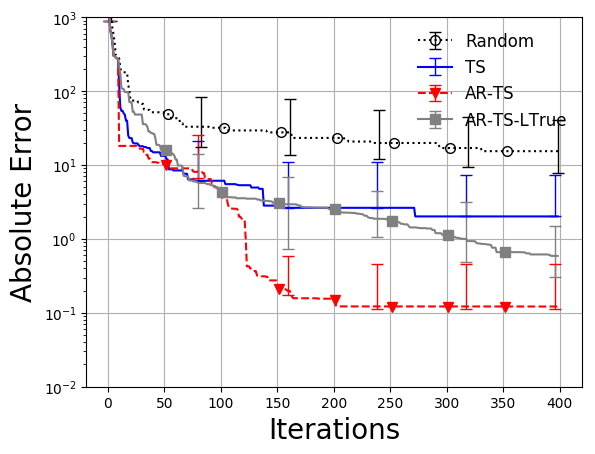}
		\label{fig:rosenbrock-4D-TSa}
	}
	\subfigure[Rosenbrock 5D]
	{
		\includegraphics[width=0.31\textwidth,height=0.28\textwidth]{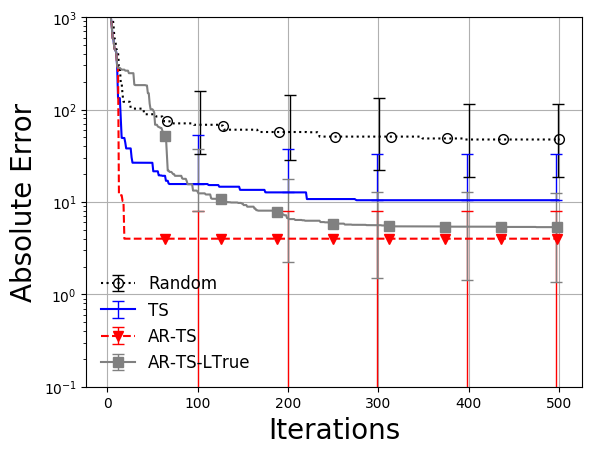}
		\label{fig:rosenbrock-5D-TSa}
	}	
		\\	
	\subfigure[Robot pushing 3D]
	{
		\includegraphics[width=0.31\textwidth,height=0.28\textwidth]{neg_robot_pushing_3D-3-TS.png}
		\label{fig:robot-3D-TSa}
	}	
	\subfigure[Robot pushing 4D]
	{
		\includegraphics[width=0.31\textwidth,height=0.28\textwidth]{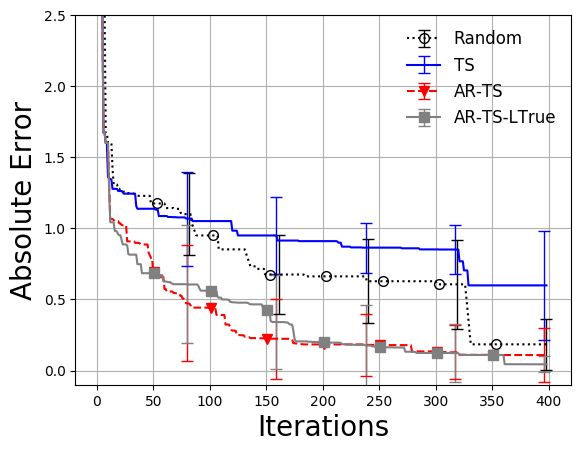}
		\label{fig:robot-4D-TSa}
	}
	\subfigure[Logistic Regression]
	{
		\includegraphics[width=0.31\textwidth,height=0.28\textwidth]{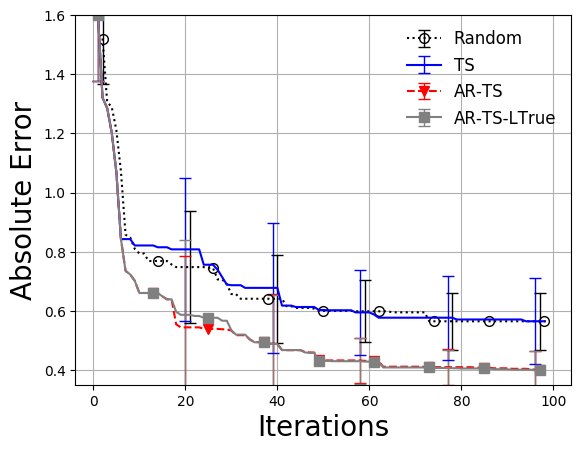}
		\label{fig:logistic-TSa}
	}	
	
	\caption{Comparing the performance of the conventional BO acquisition function, corresponding LBO mixed acquisition function, Lipschitz optimization and random exploration for the TS acquisition functions.}
	\label{fig:all-TS}
\end{figure*}

\begin{figure*}[!ht]
	\centering
	\subfigure[Branin 2D]
	{
		\includegraphics[width=0.31\textwidth,height=0.28\textwidth]{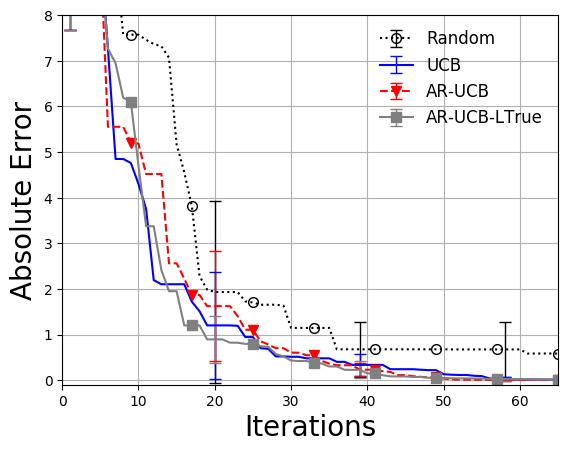}
		\label{fig:branin-2D-UCBa}
	}
	\subfigure[Camel 2D]
	{
		\includegraphics[width=0.31\textwidth,height=0.28\textwidth]{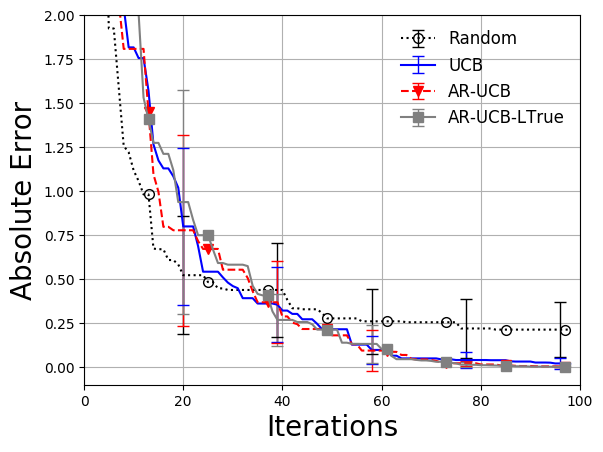}
		\label{fig:camel-2D-UCBa}
	}
	\subfigure[Goldstein Price 2D]
	{
		\includegraphics[width=0.31\textwidth,height=0.28\textwidth]{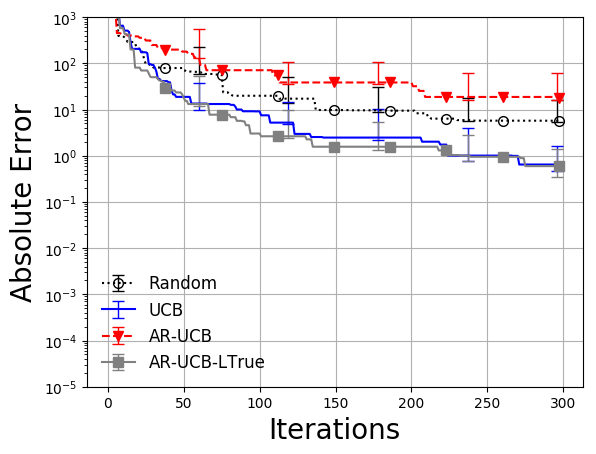}
		\label{fig:goldstein-2D-UCBa}
	}
	\\
	\subfigure[Michalwicz 2D]
	{
		\includegraphics[width=0.31\textwidth,height=0.28\textwidth]{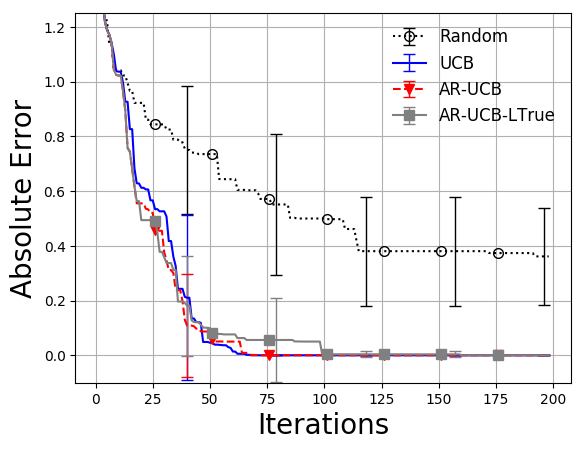}
		\label{fig:mich-2D-UCBa}
	}
	\subfigure[Michalwicz 5D]
	{
		\includegraphics[width=0.31\textwidth,height=0.28\textwidth]{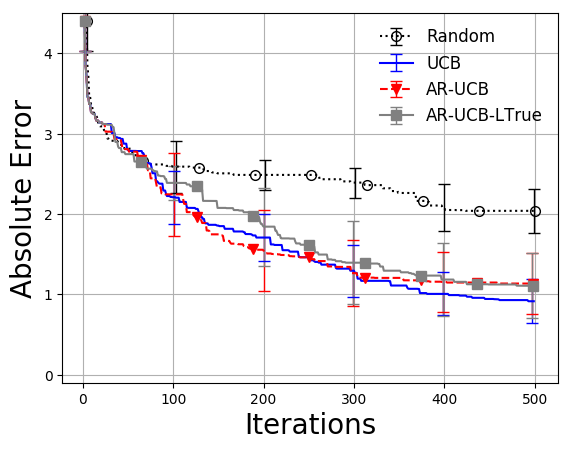}
		\label{fig:mich-5D-UCBa}
	}
	\subfigure[Michalwicz 10D]
	{
		\includegraphics[width=0.31\textwidth,height=0.28\textwidth]{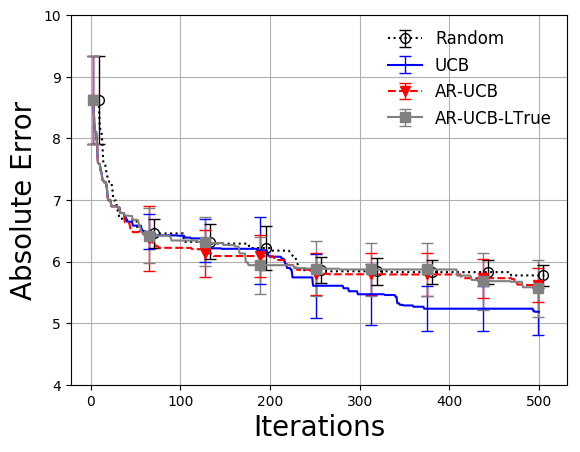}
		\label{fig:mich-10D-UCBa}
	}
	\\
	\subfigure[Rosenbrock 2D]
	{
		\includegraphics[width=0.31\textwidth,height=0.28\textwidth]{neg_rosen-2-UCB.png}
		\label{fig:rosenbrock-2D-UCBa}
	}
	\subfigure[Hartmann 3D]
	{
		\includegraphics[width=0.31\textwidth,height=0.28\textwidth]{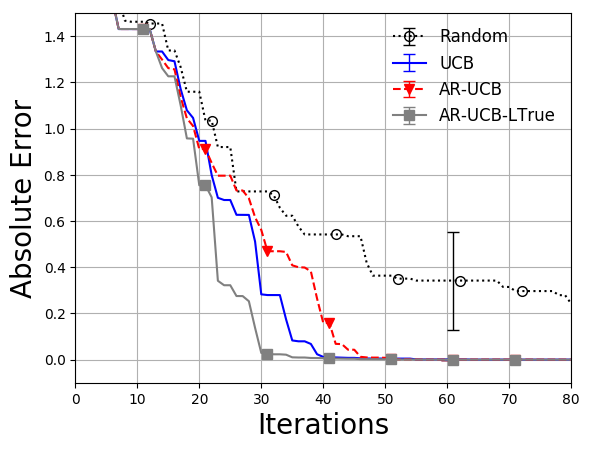}
		\label{fig:hartmann-3D-UCBa}
	}
	\subfigure[Hartmann 6D]
	{
		\includegraphics[width=0.31\textwidth,height=0.28\textwidth]{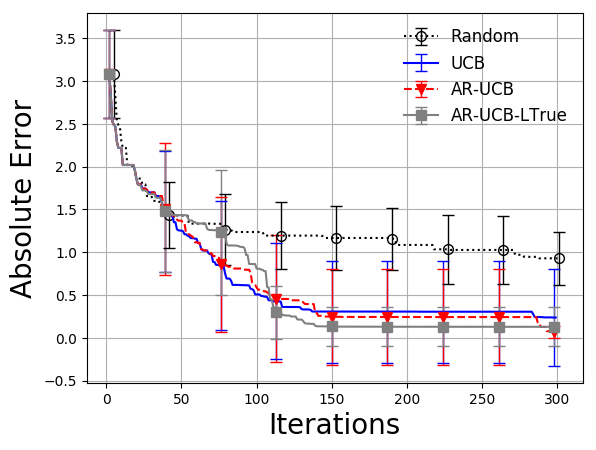}
		\label{fig:hartmann-6D-UCBa}
	}
	\\
	\subfigure[Rosenbrock 3D]
	{
		\includegraphics[width=0.31\textwidth,height=0.28\textwidth]{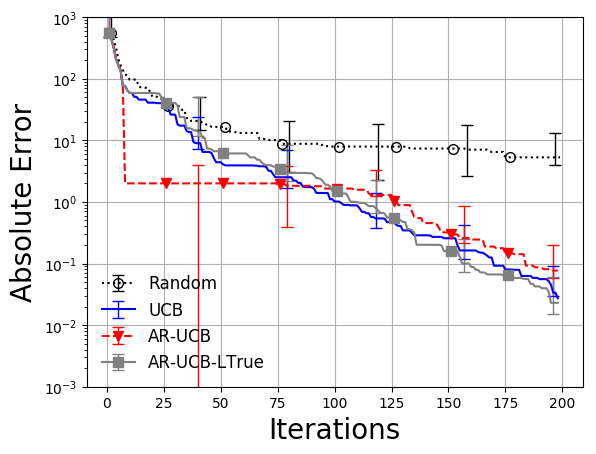}
		\label{fig:rosenbrock-3D-UCBa}
	}
	\subfigure[Rosenbrock 4D]
	{
		\includegraphics[width=0.31\textwidth,height=0.28\textwidth]{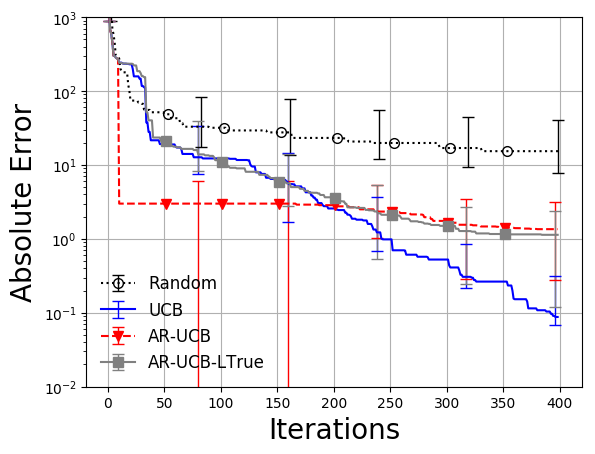}
		\label{fig:rosenbrock-4D-UCBa}
	}
	\subfigure[Rosenbrock 5D]
	{
		\includegraphics[width=0.31\textwidth,height=0.28\textwidth]{neg_rosen-5-UCB.png}
		\label{fig:rosenbrock-5D-UCBa}
	}
	\\
	\subfigure[Robot pushing 3D]
	{
		\includegraphics[width=0.31\textwidth,height=0.28\textwidth]{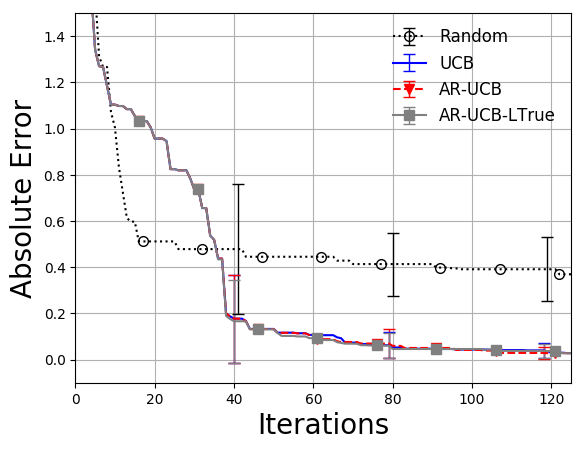}
		\label{fig:robot-3D-UCBa}
	}	
	\subfigure[Robot pushing 4D]
	{
		\includegraphics[width=0.31\textwidth,height=0.28\textwidth]{neg_robot_pushing_4D-4-UCB.png}
		\label{fig:robot-4D-UCBa}
	}
	\subfigure[Logistic Regression]
	{
		\includegraphics[width=0.31\textwidth,height=0.28\textwidth]{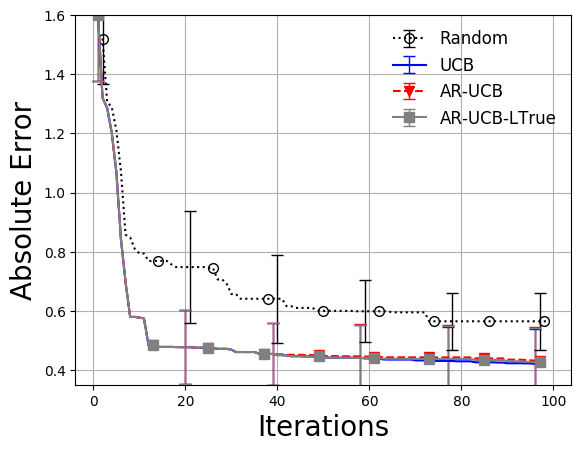}
		\label{fig:logistic-UCBa}
	}

	\caption{Comparing the performance of the conventional BO acquisition function, corresponding LBO mixed acquisition function, Lipschitz optimization and random exploration for the UCB acquisition functions.}
	\label{fig:all-UCB}
\end{figure*}

\begin{figure*}[!ht]
	\centering
	\subfigure[Branin 2D]
	{
		\includegraphics[width=0.31\textwidth,height=0.28\textwidth]{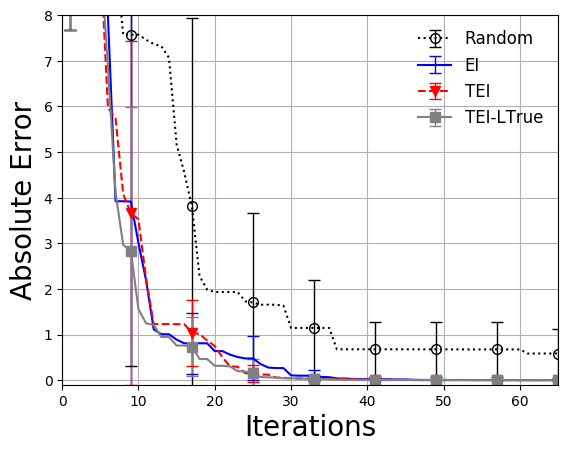}
		\label{fig:branin-2D-EIa}
	}
	\subfigure[Camel 2D]
	{
		\includegraphics[width=0.31\textwidth,height=0.28\textwidth]{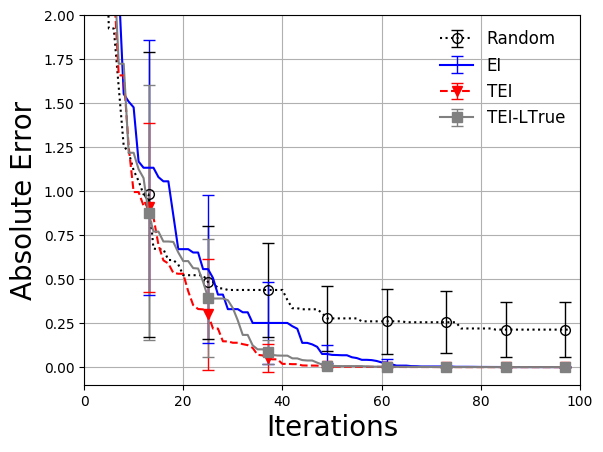}
		\label{fig:camel-2D-EIa}
	}
	\subfigure[Goldstein Price 2D]
	{
		\includegraphics[width=0.31\textwidth,height=0.28\textwidth]{neg_goldsteinPrice-2-EI.png}
		\label{fig:goldstein-2D-EIa}
	}
\\
	\subfigure[Michalwicz 2D]
	{
		\includegraphics[width=0.31\textwidth,height=0.28\textwidth]{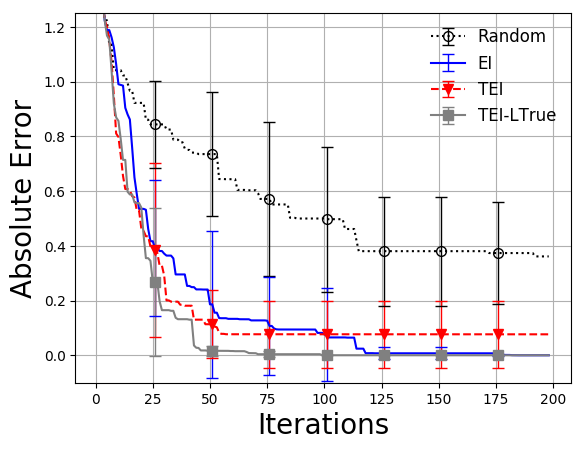}
		\label{fig:mich-2D-EIa}
	}
	\subfigure[Michalwicz 5D]
	{
		\includegraphics[width=0.31\textwidth,height=0.28\textwidth]{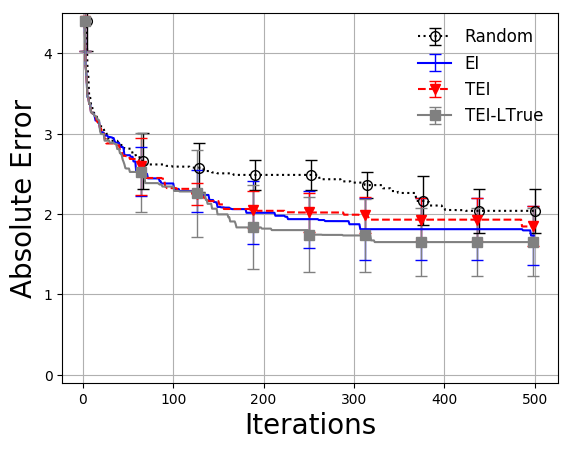}
		\label{fig:mich-5D-EIa}
	}
	\subfigure[Michalwicz 10D]
	{
		\includegraphics[width=0.31\textwidth,height=0.28\textwidth]{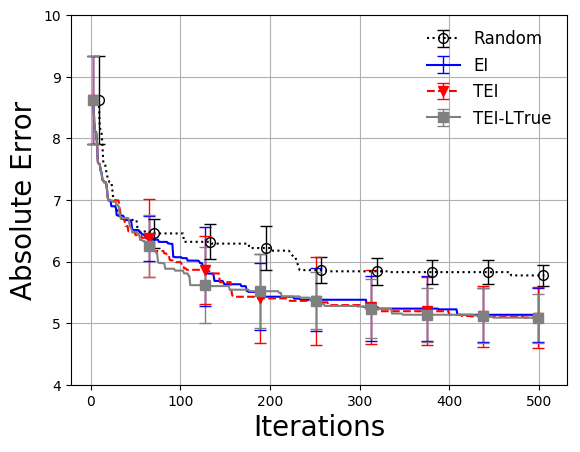}
		\label{fig:mich-10D-EIa}
	}
	\\
	\subfigure[Rosenbrock 2D]
	{
		\includegraphics[width=0.31\textwidth,height=0.28\textwidth]{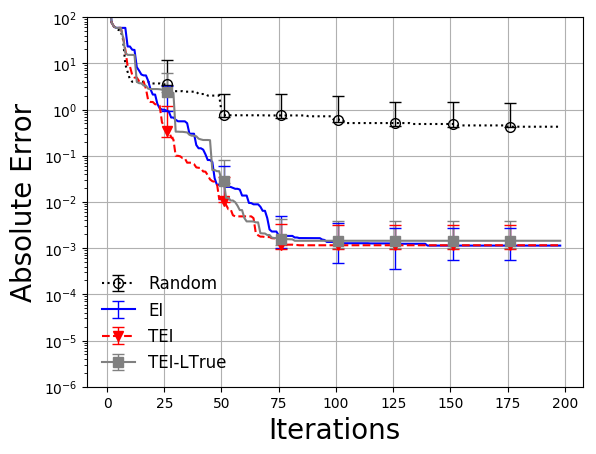}
		\label{fig:rosenbrock-2D-EIa}
	}
	\subfigure[Hartmann 3D]
	{
		\includegraphics[width=0.31\textwidth,height=0.28\textwidth]{neg_hartmann3-3-EI.png}
		\label{fig:hartmann-3D-EIa}
	}
	\subfigure[Hartmann 6D]
	{
		\includegraphics[width=0.31\textwidth,height=0.28\textwidth]{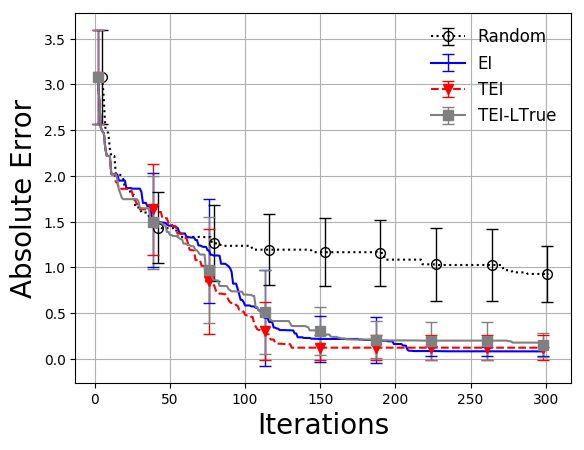}
		\label{fig:hartmann-6D-EIa}
	}
	\\
	\subfigure[Rosenbrock 3D]
	{
	\includegraphics[width=0.31\textwidth,height=0.28\textwidth]{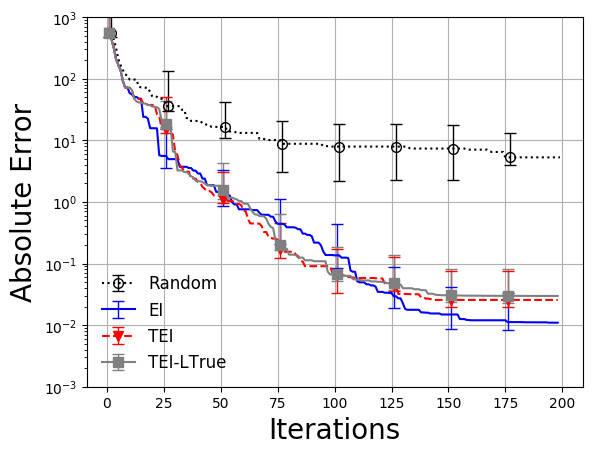}
	\label{fig:rosenbrock-3D-EIa}
	}
	\subfigure[Rosenbrock 4D]
	{
	\includegraphics[width=0.31\textwidth,height=0.28\textwidth]{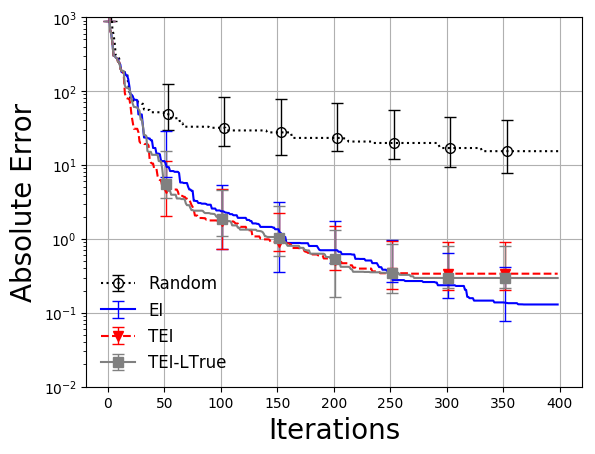}
	\label{fig:rosenbrock-4D-EIa}
	}
	\subfigure[Rosenbrock 5D]
	{
	\includegraphics[width=0.31\textwidth,height=0.28\textwidth]{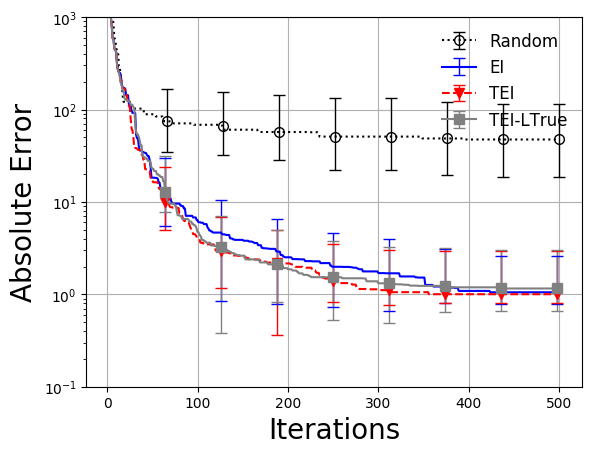}
	\label{fig:rosenbrock-5D-EIa}
	}
	\\	
	\subfigure[Robot pushing 3D]
	{
		\includegraphics[width=0.31\textwidth,height=0.28\textwidth]{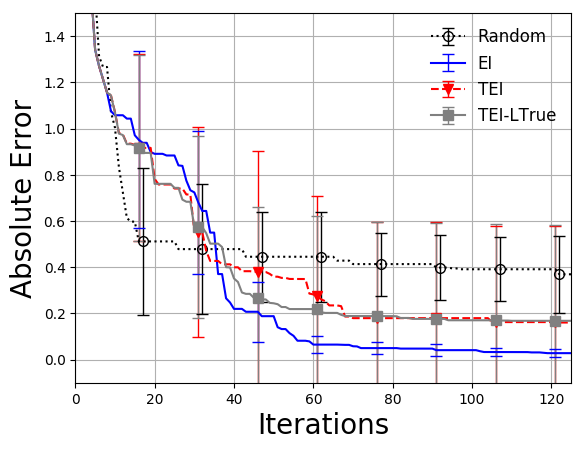}
		\label{fig:robot-3D-EIa}
	}	
	\subfigure[Robot pushing 4D]
	{
		\includegraphics[width=0.31\textwidth,height=0.28\textwidth]{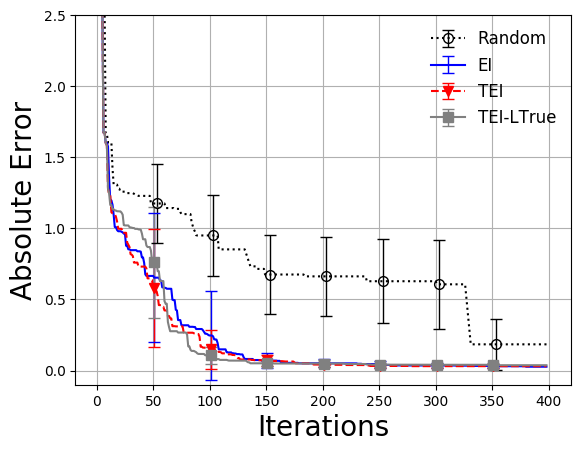}
		\label{fig:robot-4D-EIa}
	}
	\subfigure[Logistic Regression]
	{
		\includegraphics[width=0.31\textwidth,height=0.28\textwidth]{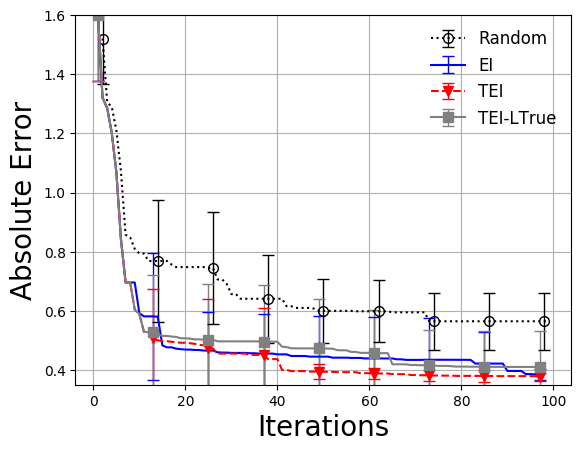}
		\label{fig:logistic-EIa}
	}
	
	\caption{Comparing the performance of the conventional BO acquisition function, corresponding LBO mixed acquisition function, Lipschitz optimization and random exploration for the EI acquisition functions.}
	\label{fig:all-EI}
\end{figure*}

\begin{figure*}[!ht]
	\centering
	\subfigure[Branin 2D]
	{
		\includegraphics[width=0.31\textwidth,height=0.28\textwidth]{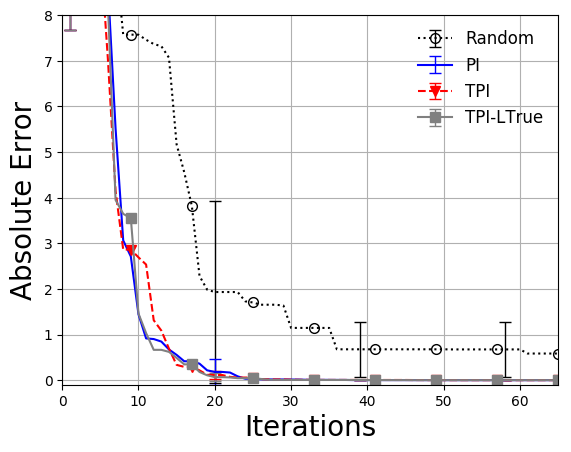}
		\label{fig:branin-2D-PIa}
	}
	\subfigure[Camel 2D]
	{
		\includegraphics[width=0.31\textwidth,height=0.28\textwidth]{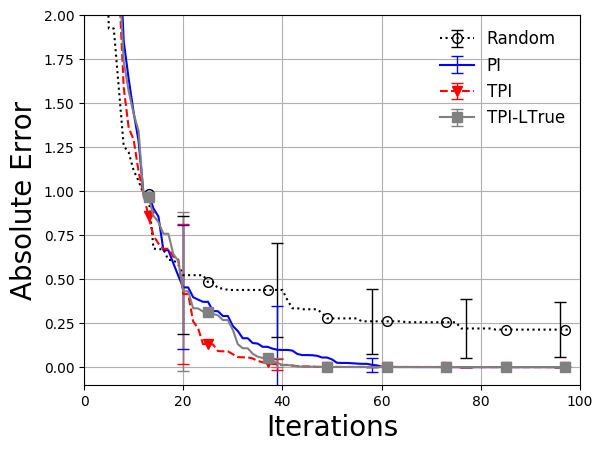}
		\label{fig:camel-2D-PIa}
	}
	\subfigure[Goldstein Price 2D]
	{
		\includegraphics[width=0.31\textwidth,height=0.28\textwidth]{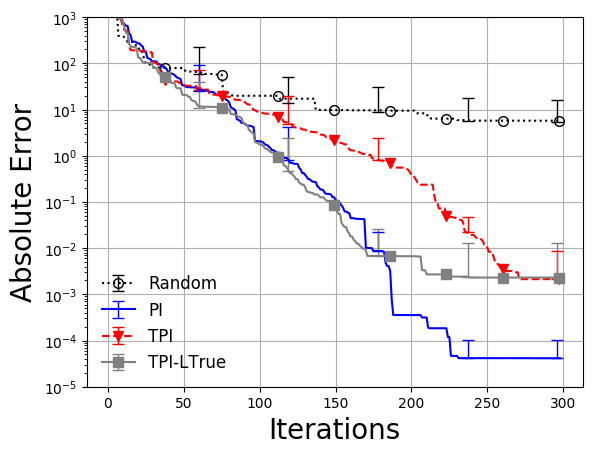}
		\label{fig:goldstein-2D-PIa}
	}
	\\
	\subfigure[Michalwicz 2D]
	{
		\includegraphics[width=0.31\textwidth,height=0.28\textwidth]{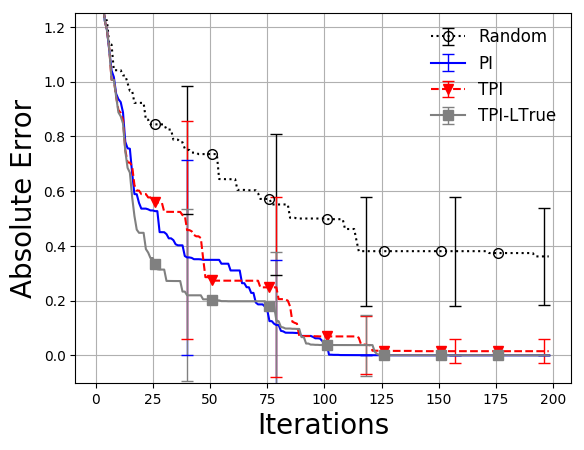}
		\label{fig:mich-2D-PIa}
	}
	\subfigure[Michalwicz 5D]
	{
		\includegraphics[width=0.31\textwidth,height=0.28\textwidth]{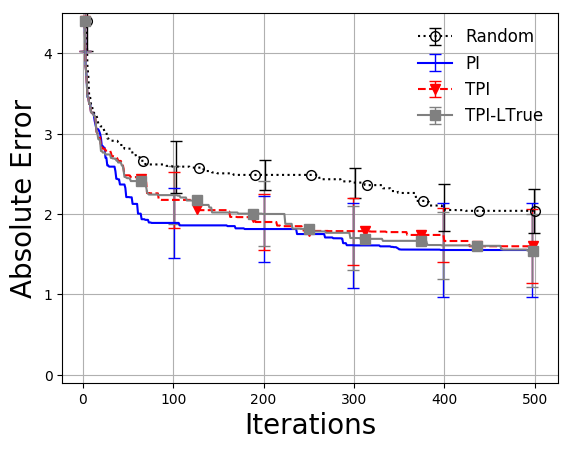}
		\label{fig:mich-5D-PIa}
	}
	\subfigure[Michalwicz 10D]
	{
		\includegraphics[width=0.31\textwidth,height=0.28\textwidth]{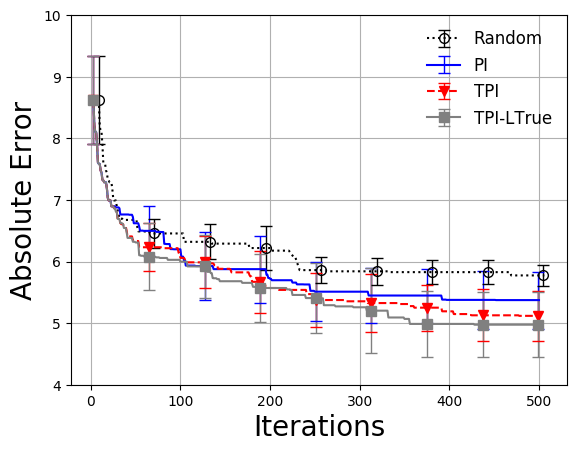}
		\label{fig:mich-10D-PIa}
	}
	\\
	\subfigure[Rosenbrock 2D]
	{
		\includegraphics[width=0.31\textwidth,height=0.28\textwidth]{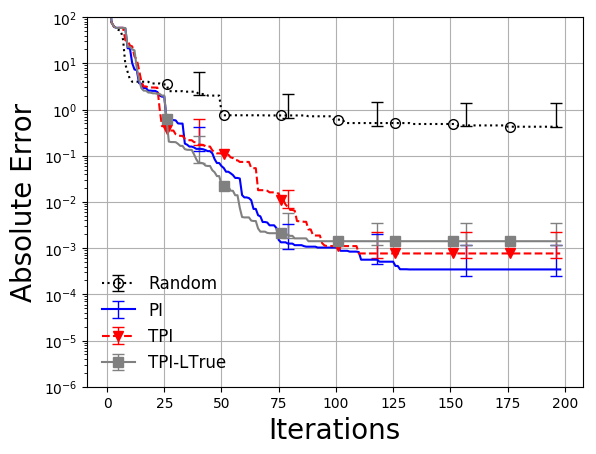}
		\label{fig:rosenbrock-2D-PIa}
	}
	\subfigure[Hartmann 3D]
	{
		\includegraphics[width=0.31\textwidth,height=0.28\textwidth]{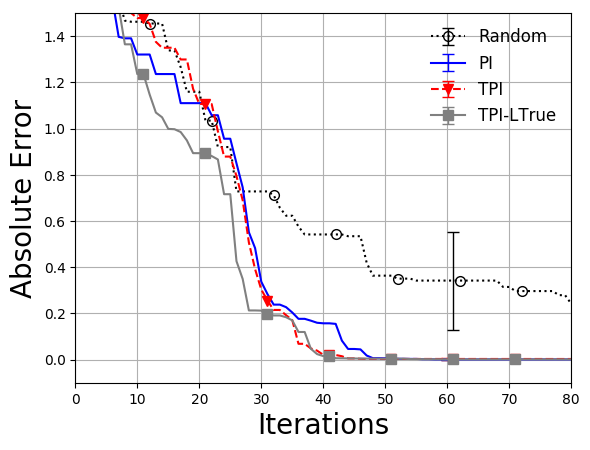}
		\label{fig:hartmann-3D-PIa}
	}
	\subfigure[Hartmann 6D]
	{
		\includegraphics[width=0.31\textwidth,height=0.28\textwidth]{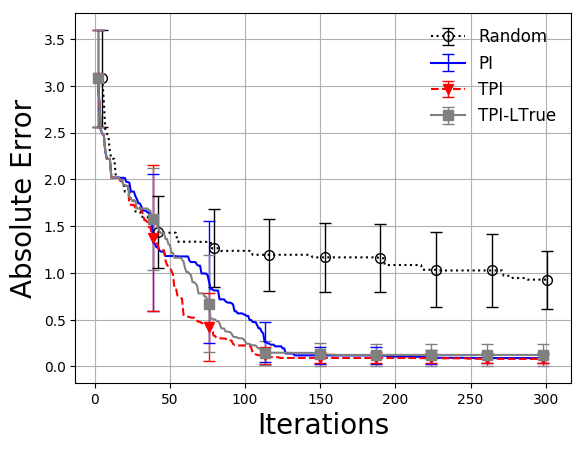}
		\label{fig:hartmann-6D-PIa}
	}
	\\
	\subfigure[Rosenbrock 3D]
	{
		\includegraphics[width=0.31\textwidth,height=0.28\textwidth]{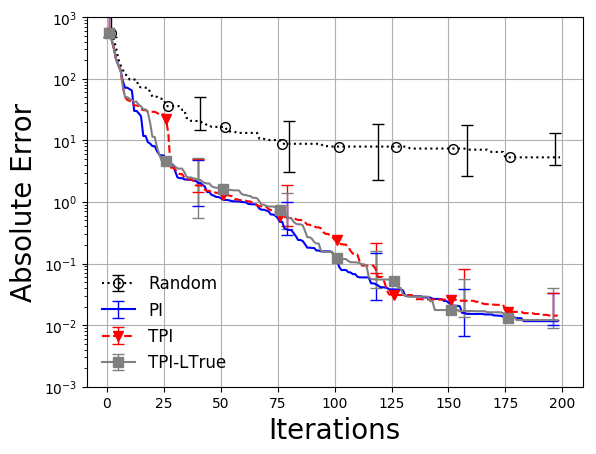}
		\label{fig:rosenbrock-3D-PIa}
	}
	\subfigure[Rosenbrock 4D]
	{
		\includegraphics[width=0.31\textwidth,height=0.28\textwidth]{neg_rosen-4-PI.png}
		\label{fig:rosenbrock-4D-PIa}
	}
	\subfigure[Rosenbrock 5D]
	{
		\includegraphics[width=0.31\textwidth,height=0.28\textwidth]{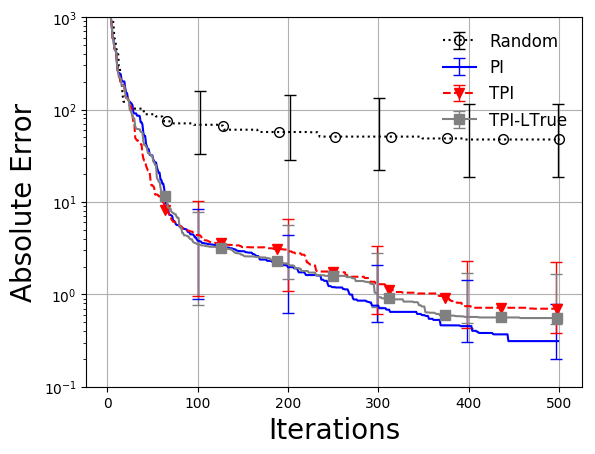}
		\label{fig:rosenbrock-5D-PIa}
	}	
	\\	
	\subfigure[Robot pushing 3D]
	{
		\includegraphics[width=0.31\textwidth,height=0.28\textwidth]{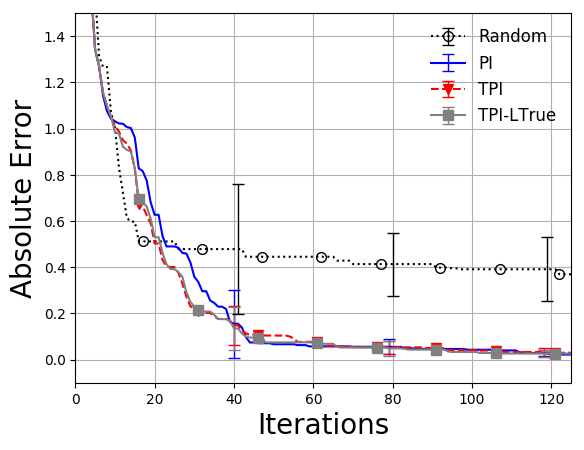}
		\label{fig:robot-3D-PIa}
	}	
	\subfigure[Robot pushing 4D]
	{
		\includegraphics[width=0.31\textwidth,height=0.28\textwidth]{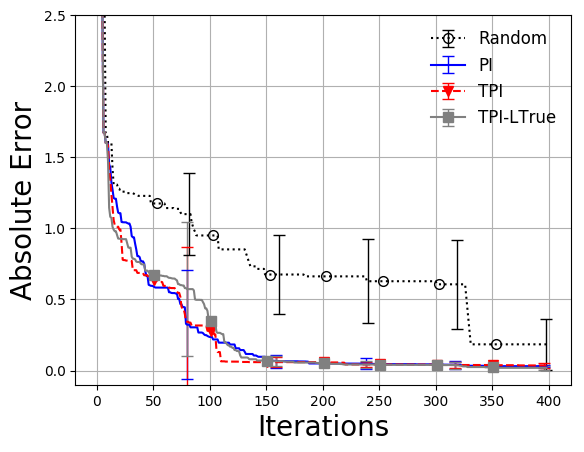}
		\label{fig:robot-4D-PIa}
	}
	\subfigure[Logistic Regression]
	{
		\includegraphics[width=0.31\textwidth,height=0.28\textwidth]{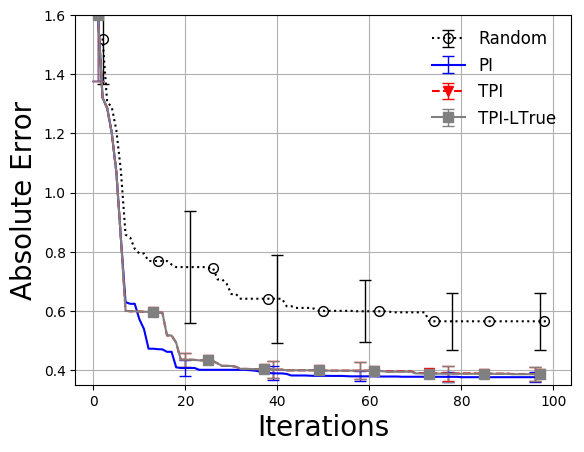}
		\label{fig:logistic-PIa}
	}	
	
	\caption{Comparing the performance of the conventional BO acquisition function, corresponding LBO mixed acquisition function, Lipschitz optimization and random exploration for the PI acquisition functions.}
	\label{fig:all-PI}
\end{figure*}

\begin{figure*}[!ht]
	\centering
	\subfigure[Branin 2D]
	{
		\includegraphics[width=0.31\textwidth,height=0.28\textwidth]{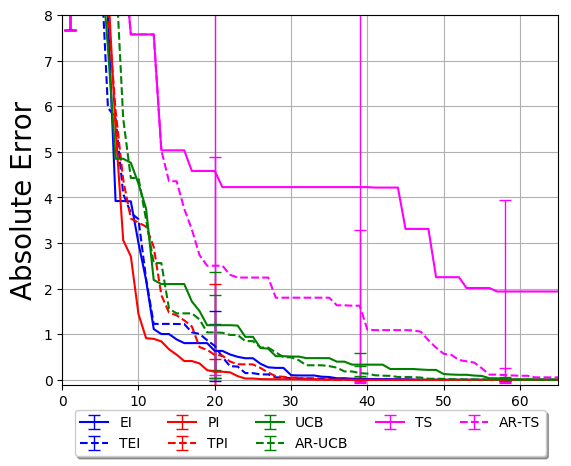}
		\label{fig:branin-2Da}
	}
	\subfigure[Camel 2D]
	{
		\includegraphics[width=0.31\textwidth,height=0.28\textwidth]{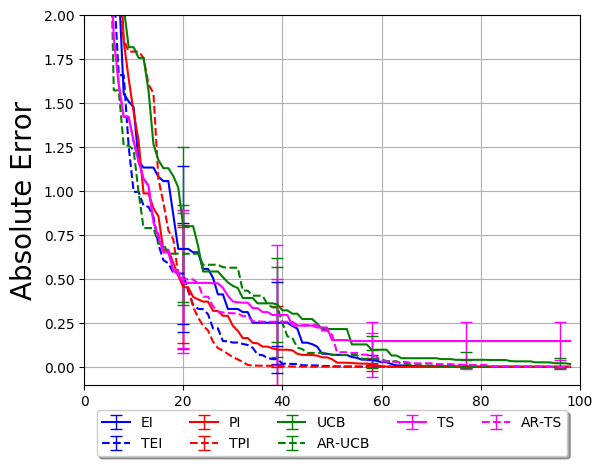}
		\label{fig:camel-2Da}
	}
	\subfigure[Goldstein Price 2D]
	{
		\includegraphics[width=0.31\textwidth,height=0.28\textwidth]{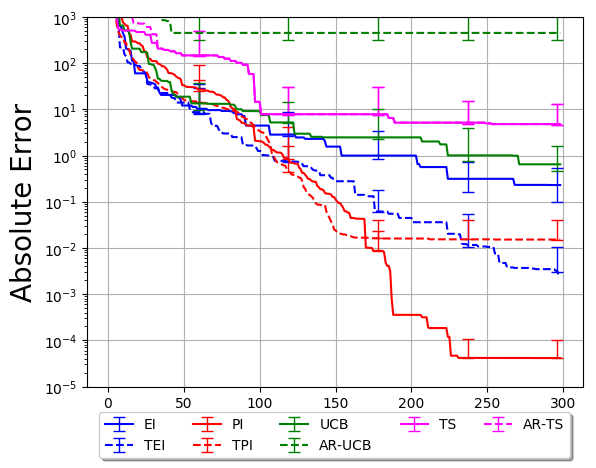}
		\label{fig:goldstein-2Da}
	}
	\\
	\subfigure[Michalwicz 2D]
	{
		\includegraphics[width=0.31\textwidth,height=0.28\textwidth]{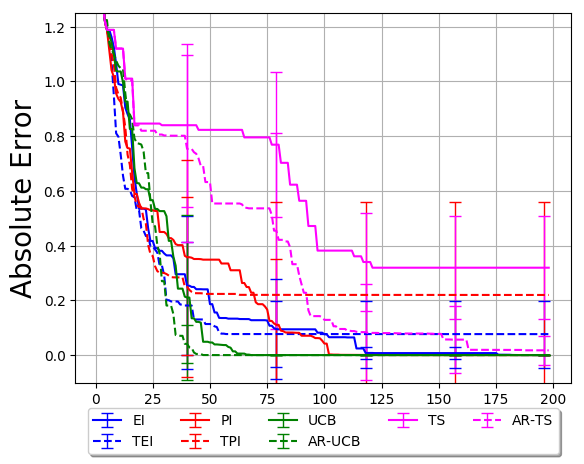}
		\label{fig:mich-2Da}
	}
	\subfigure[Michalwicz 5D]
	{
		\includegraphics[width=0.31\textwidth,height=0.28\textwidth]{neg_michalewicz-5.png}
		\label{fig:mich-5Da}
	}
	\subfigure[Michalwicz 10D]
	{
		\includegraphics[width=0.31\textwidth,height=0.28\textwidth]{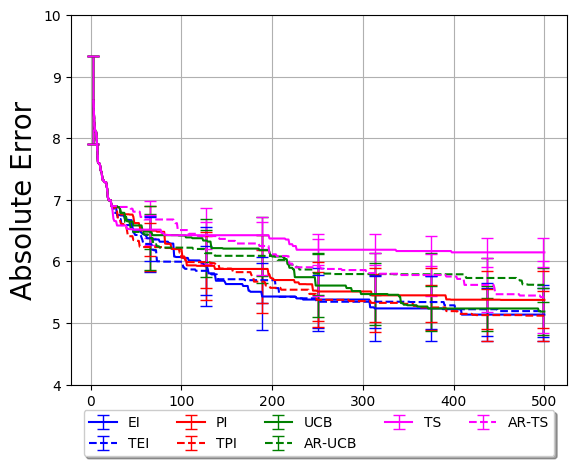}
		\label{fig:mich-10Da}
	}
	\\
	\subfigure[Rosenbrock 2D]
	{
		\includegraphics[width=0.31\textwidth,height=0.28\textwidth]{neg_rosen-2.png}
		\label{fig:rosenbrock-2Da}
	}
	\subfigure[Hartmann 3D]
	{
		\includegraphics[width=0.31\textwidth,height=0.28\textwidth]{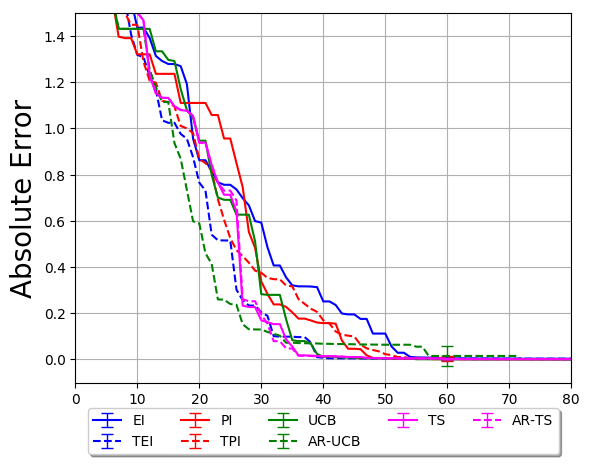}
		\label{fig:hartmann-3Da}
	}
	\subfigure[Hartmann 6D]
	{
		\includegraphics[width=0.31\textwidth,height=0.28\textwidth]{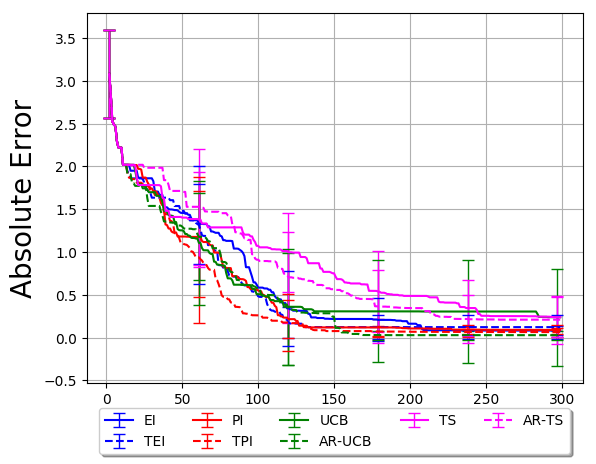}
		\label{fig:hartmann-6Da}
	}
	\\
	\subfigure[Rosenbrock 3D]
	{
		\includegraphics[width=0.31\textwidth,height=0.28\textwidth]{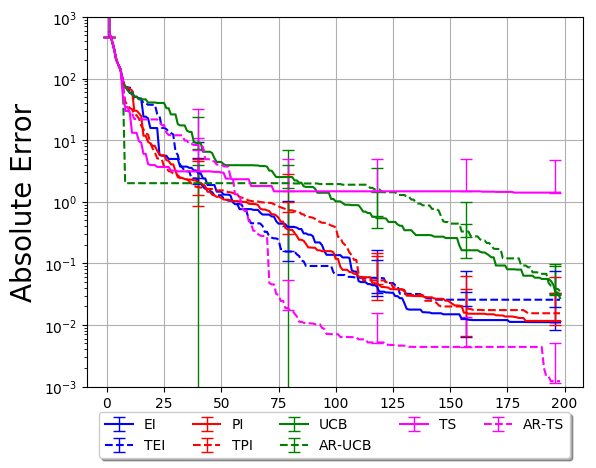}
		\label{fig:rosenbrock-3Da}
	}
	\subfigure[Rosenbrock 4D]
	{
		\includegraphics[width=0.31\textwidth,height=0.28\textwidth]{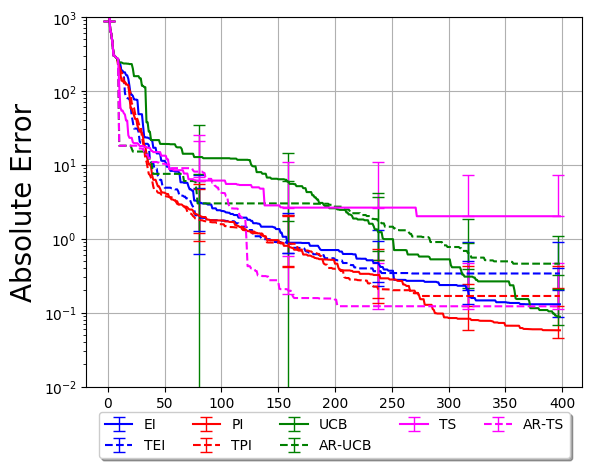}
		\label{fig:rosenbrock-4Da}
	}
	\subfigure[Rosenbrock 5D]
	{
		\includegraphics[width=0.31\textwidth,height=0.28\textwidth]{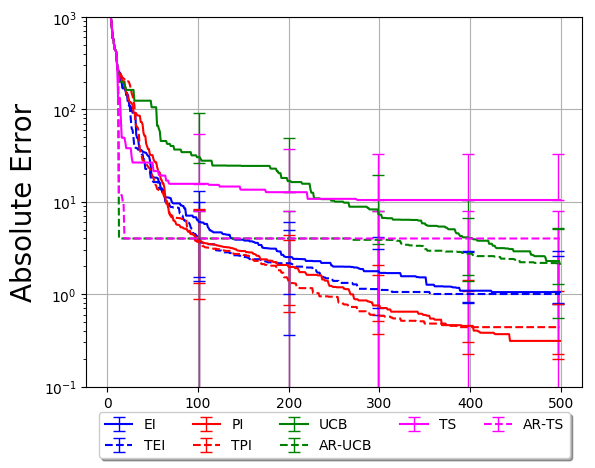}
		\label{fig:rosenbrock-5Da}
	}
\\

\subfigure[Robot pushing 3D]
{
	\includegraphics[width=0.31\textwidth,height=0.28\textwidth]{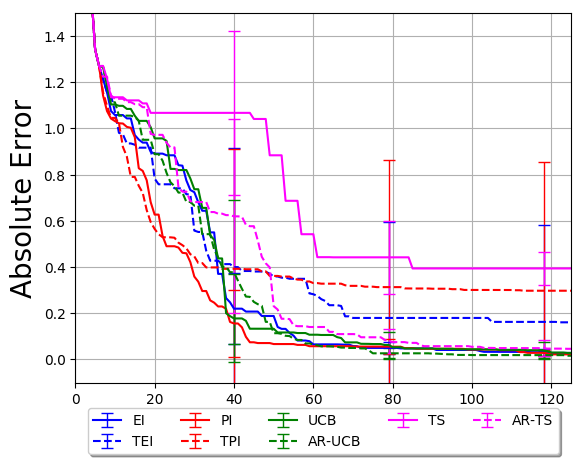}
	\label{fig:robot-3D-PIa}
}	
\subfigure[Robot pushing 4D]
{
	\includegraphics[width=0.31\textwidth,height=0.28\textwidth]{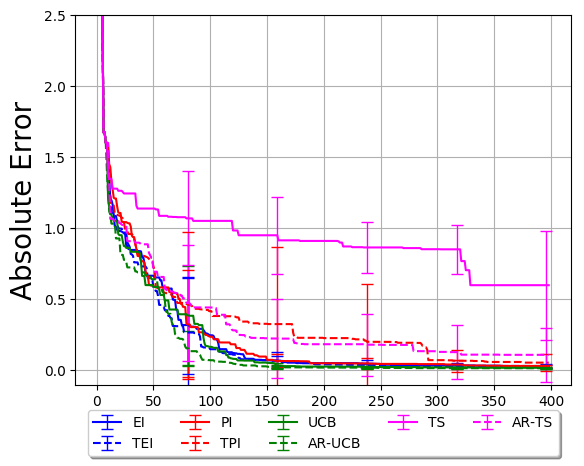}
	\label{fig:robot-4D-PIa}
}
\subfigure[Logistic Regression]
{
	\includegraphics[width=0.31\textwidth,height=0.28\textwidth]{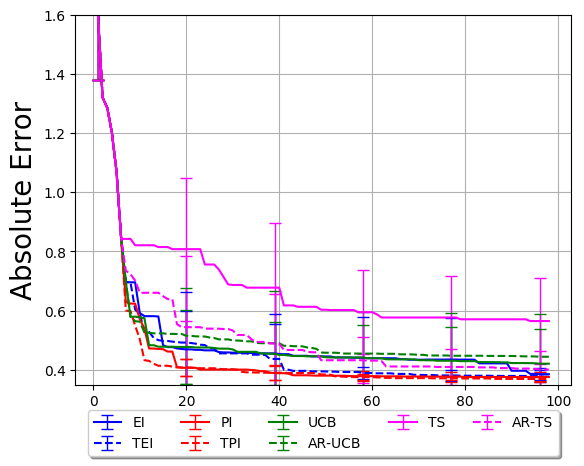}
	\label{fig:logistic-PIa}
}

	\caption{Comparing the performance across the four BO and the corresponding LBO acquisition functions against Lipschitz optimization and random exploration on all the test functions (Better seen in color).}
	\label{fig:all-joint}
\end{figure*}

\begin{figure*}[!ht]
	\centering
	\subfigure[Branin 2D]
	{
		\includegraphics[width=0.31\textwidth,height=0.28\textwidth]{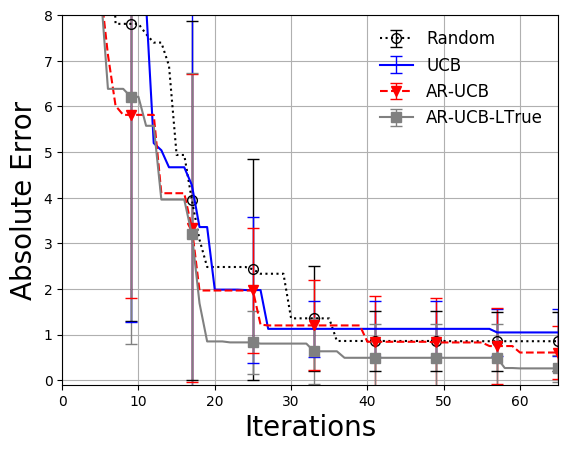}
		\label{fig:branin-2D-UCBa}
	}
	\subfigure[Camel 2D]
	{
		\includegraphics[width=0.31\textwidth,height=0.28\textwidth]{neg_camel-2-UCB_large.png}
		\label{fig:camel-2D-UCBa}
	}
	\subfigure[Goldstein Price 2D]
	{
		\includegraphics[width=0.31\textwidth,height=0.28\textwidth]{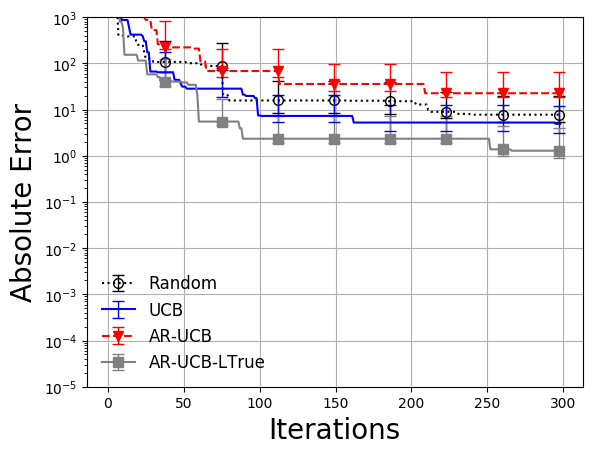}
		\label{fig:goldstein-2D-UCBa}
	}
	\\
	\subfigure[Michalwicz 2D]
	{
		\includegraphics[width=0.31\textwidth,height=0.28\textwidth]{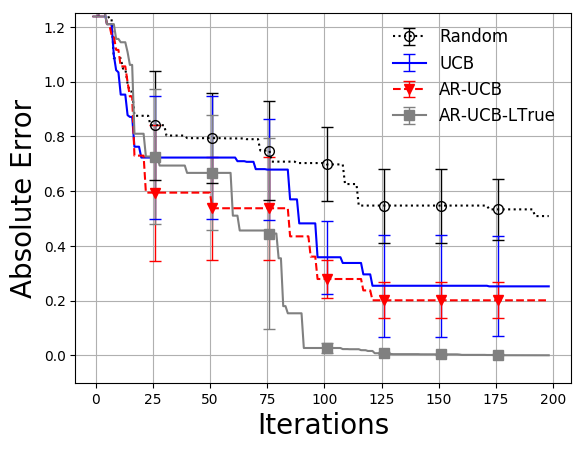}
		\label{fig:mich-2D-UCBa}
	}
	\subfigure[Michalwicz 5D]
	{
		\includegraphics[width=0.31\textwidth,height=0.28\textwidth]{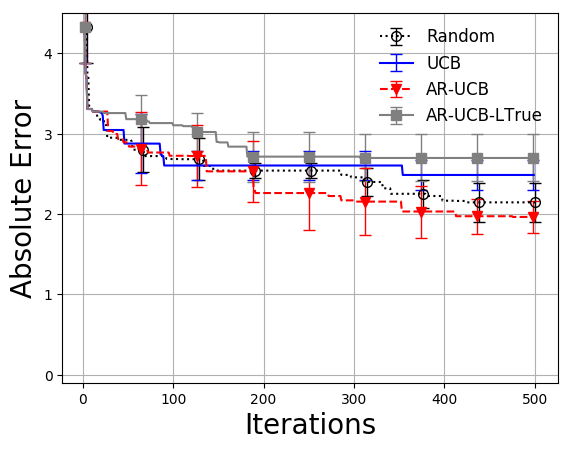}
		\label{fig:mich-5D-UCBa}
	}
	\subfigure[Michalwicz 10D]
	{
		\includegraphics[width=0.31\textwidth,height=0.28\textwidth]{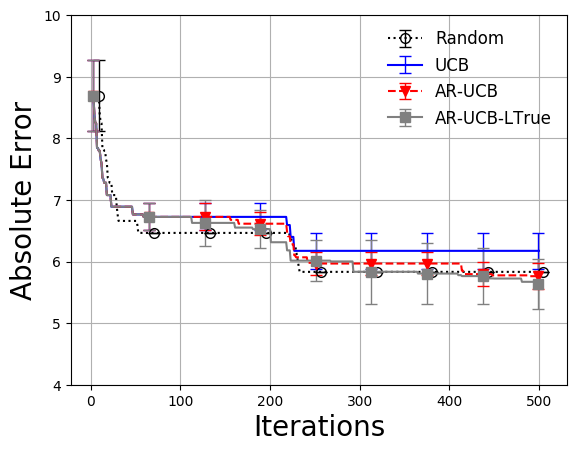}
		\label{fig:mich-10D-UCBa}
	}
	\\
	\subfigure[Rosenbrock 2D]
	{
		\includegraphics[width=0.31\textwidth,height=0.28\textwidth]{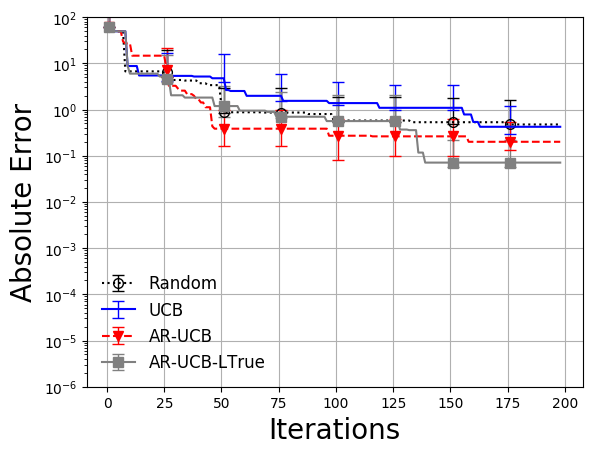}
		\label{fig:rosenbrock-2D-UCBa}
	}
	\subfigure[Hartmann 3D]
	{
		\includegraphics[width=0.31\textwidth,height=0.28\textwidth]{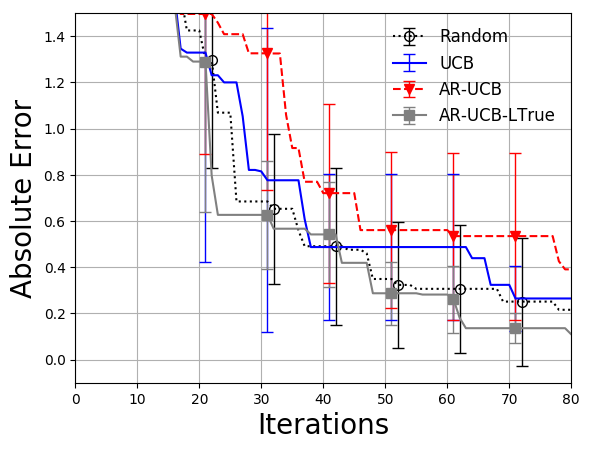}
		\label{fig:hartmann-3D-UCBa}
	}
	\subfigure[Hartmann 6D]
	{
		\includegraphics[width=0.31\textwidth,height=0.28\textwidth]{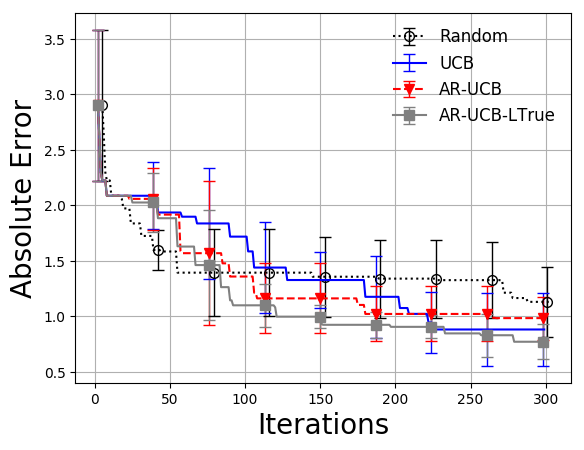}
		\label{fig:hartmann-6D-UCBa}
	}
	\\
	\subfigure[Rosenbrock 3D]
	{
		\includegraphics[width=0.31\textwidth,height=0.28\textwidth]{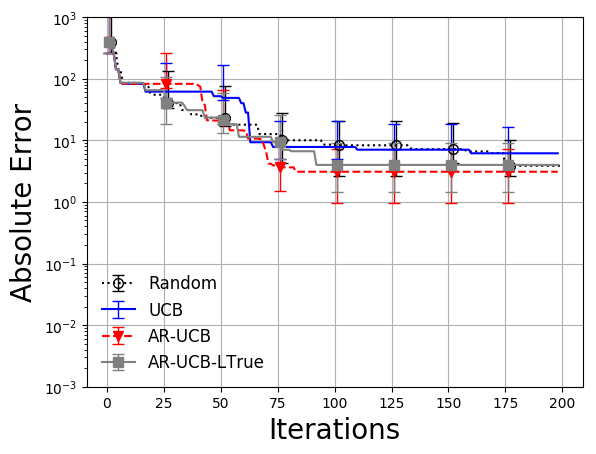}
		\label{fig:rosenbrock-3D-UCBa}
	}
	\subfigure[Rosenbrock 4D]
	{
		\includegraphics[width=0.31\textwidth,height=0.28\textwidth]{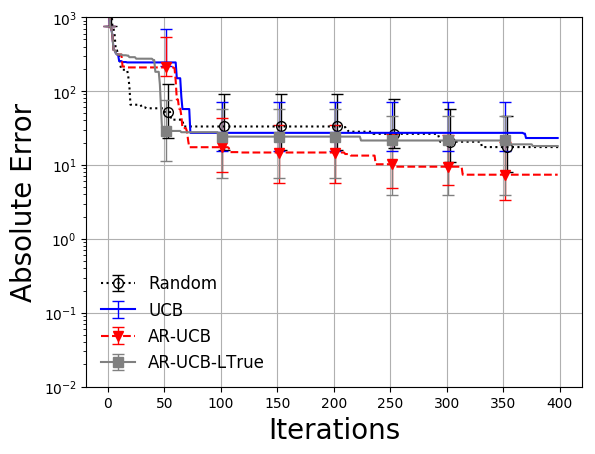}
		\label{fig:rosenbrock-4D-UCBa}
	}
	\subfigure[Rosenbrock 5D]
	{
		\includegraphics[width=0.31\textwidth,height=0.28\textwidth]{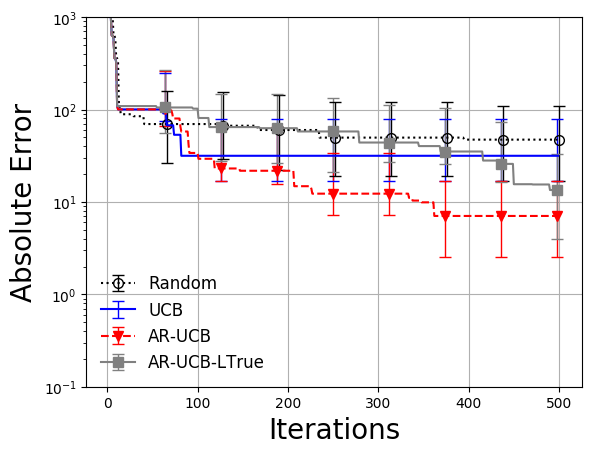}
		\label{fig:rosenbrock-5D-UCBa}
	}
	\\
	\subfigure[Robot pushing 3D]
	{
		\includegraphics[width=0.31\textwidth,height=0.28\textwidth]{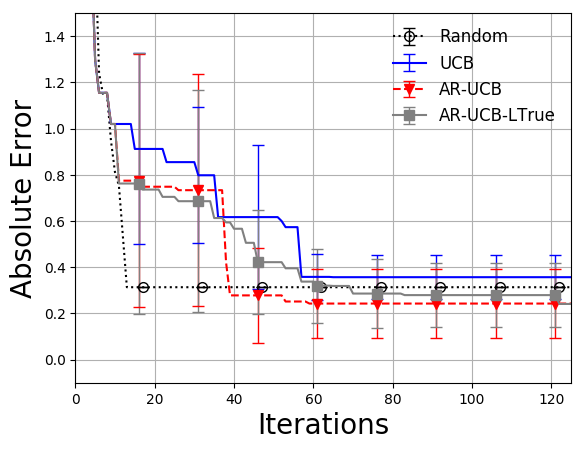}
		\label{fig:robot-3D-UCBa}
	}	
	\subfigure[Robot pushing 4D]
	{
		\includegraphics[width=0.31\textwidth,height=0.28\textwidth]{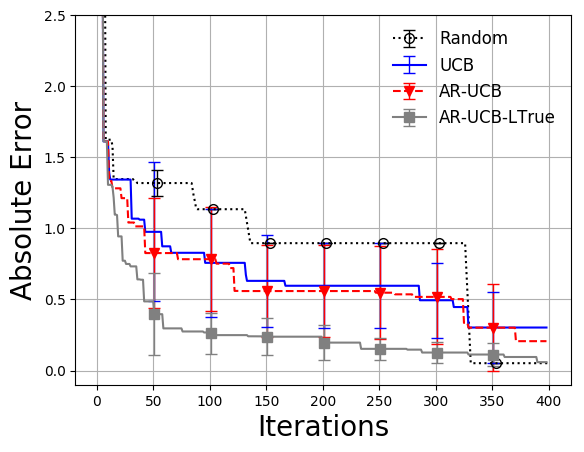}
		\label{fig:robot-4D-UCBa}
	}
	\subfigure[Logistic Regression]
	{
		\includegraphics[width=0.31\textwidth,height=0.28\textwidth]{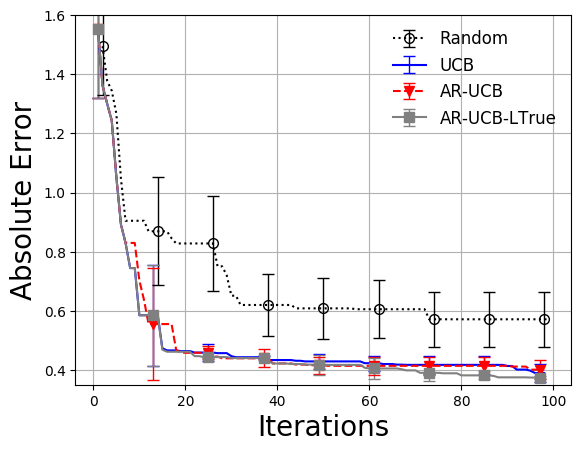}
		\label{fig:logistic-UCBa}
	}

	\caption{Comparing the performance of the conventional BO acquisition function, corresponding LBO mixed acquisition function, Lipschitz optimization and random exploration for the UCB acquisition functions when using very large $\beta = 10^{16}$}
	\label{fig:all-large-UCB}
\end{figure*}

\end{document}